\documentclass{article}

\usepackage{arxiv}

\usepackage[utf8]{inputenc} 
\usepackage{hyperref}       
\usepackage{url}            
\usepackage{booktabs}       
\usepackage{amsfonts}       
\usepackage{amsmath}
\usepackage{amsthm}
\usepackage{mathabx}
\usepackage{multirow}
\usepackage{booktabs}
\usepackage{float} 
\usepackage{nicefrac}       
\usepackage{microtype}      
\usepackage{lipsum}
\usepackage{graphicx,subfigure}
\usepackage{paralist}
\usepackage{algorithm}
\usepackage{algorithmic}
\usepackage{breqn}
\usepackage{comment}
\graphicspath{ {./images/} }

\newcommand{\gpPred}{\hat{Y}^{\gamma}_{ijk}}
\newcommand{\gpVar}{s^{2,\gamma}_{ijk}}
\newcommand{\Significance}{\delta_{\tiny{C}}}
\newcommand{\falsq}{\delta_q}
\newcommand{\minVol}{\delta_{\tiny{v}}}

\newcommand{\FalsVol}{\mathcal{V}^{f}}
\newcommand{\Designspace}{S}
\newcommand{\LO}{\mathcal{L}_{0}}
\newcommand{\RegionPlus}{\Theta^{+}}
\newcommand{\RegionMinus}{\Theta^{-}}
\newcommand{\RegionUnd}{\Theta^{u}}
\newcommand{\RegionR}{\Theta^{r}}
\newcommand{\Subreg}{\sigma}

\newcommand{\Var}{\mbox{Var}}
\newcommand{\mbf}{\mathbf}

\newcommand{\RegionRp}{\Theta^{r,+}}
\newcommand{\RegionRm}{\Theta^{r,-}}

\renewcommand{\hat}{\widehat}
\renewcommand{\mathbf}{\boldsymbol}
\renewcommand{\bar}{\widebar}
\newtheorem{assumption}{Assumption}
\newtheorem{definition}{Definition}

\newtheorem{remark}{Remark}

\newtheorem{lemma}{Lemma}
\newtheorem{theorem}{Theorem}
\newtheorem{corollary}{Corollary}

\title{Part-X: A Family of Stochastic Algorithms for Search-Based Test Generation with Probabilistic Guarantees}

\author{
 Giulia Pedrielli \\
  School of Computing and Augmented Intelligence\\
  Arizona State University\\
  Tempe, AZ 85251 \\
  \texttt{gpedriel@asu.edu} \\
   \And
 Tanmay Khandait \\
  School of Computing and Augmented Intelligence\\
  Arizona State University\\
  Tempe, AZ 85251 \\
  \texttt{tkhandai@asu.edu} \\
  \And
 Surdeep Chotaliya \\
  School of Computing and Augmented Intelligence\\
  Arizona State University\\
  Tempe, AZ 85251 \\
  \texttt{schotali@asu.edu} \\
  \And
 Quinn Thibeault \\
  School of Computing and Augmented Intelligence\\
  Arizona State University\\
  Tempe, AZ 85251 \\
  \texttt{qthibeau@asu.edu} \\
  \And
 Hao Huang \\
  College of Engineering, \\
  Yuan Ze University, \\
  135, Yuandong Rd, Taoyuan City, Taiwan 320\\
  \texttt{haohuang@saturn.yzu.edu.tw}\\
 \And
 Mauricio Castillo-Effen \\
  Advanced Technology Laboratories\\
  Lockheed Martin\\
  Arlington, VA 22202\\
  \texttt{mauricio.castillo-effen@lmco.com} \\
  \And
 Georgios Fainekos \\
  School of Computing and Augmented Intelligence\\
  Arizona State University\\
  Tempe, AZ 85251 \\
  \texttt{fainekos@asu.edu} \\
}

\begin{document}
\maketitle
\begin{abstract}
Requirements driven search-based testing (also known as falsification) has proven to be a practical and effective method for discovering erroneous behaviors in Cyber-Physical Systems. Despite the constant improvements on the performance and applicability of falsification methods, they all share a common characteristic. Namely, they are best-effort methods which do not provide any guarantees on the absence of erroneous behaviors (falsifiers) when the testing budget is exhausted. The absence of finite time guarantees is a major limitation which prevents falsification methods from being utilized in certification procedures. In this paper, we address the finite-time guarantees problem by developing a new stochastic algorithm. Our proposed algorithm not only estimates (bounds) the probability that falsifying behaviors exist, but also it identifies the regions where these falsifying behaviors may occur. We demonstrate the applicability of our approach on standard benchmark functions from the optimization literature and on the F16 benchmark problem. 
\end{abstract}


\section{Introduction and Motivation}
Search-based test generation (SBTG) for Cyber-Physical Systems (CPS)~\cite{KapinskiEtAl2016csm}  refers to a broad class of best-effort methods that attempt to discover system behaviors that do not satisfy a set of functional requirements. 
In other words, SBTG methods search the operating space of the system for behaviors that falsify (i.e., do not satisfy) the given requirements (also known as falsification process).
A variety of methods for SBTG have been developed that range from tree exploration methods~\cite{NahhalD07cav,DreossiDDKJD15nfm,ZhangEtAl2018cadics,PlakuKV09tacas} to black-box optimization based methods~\cite{abbas2013probabilistic,deshmukh2017testing,ZhangAH2020cadics} and related variations of these techniques~\cite{zutshi2014multiple,Waga2020hscc,MenghiEtAl2020icse,YamaguchiKDS2016fmcad}.
More recently, reinforcement learning methods \cite{AkazakiEtAl2018,LeeEtAl2020jair} have also being explored for SBTG. 
The allure of SBTG methods is that in general they do not require any information about the System under Test (SUT) and, hence, they can be easily applied to a range of challenging applications such as medical \cite{sankaranarayanan2017model}, automotive~\cite{TuncaliEtAl2018iv,DreossiDS2019jar}, and aerospace~\cite{LeeEtAl2020jair,MenghiEtAl2020icse} (see~\cite{BartocciEtAl2018survey,AbbasEtAl2018emsoft,CorsoEtAl2020arxivSurvey,KapinskiEtAl2016csm} for surveys).

A common characteristic of all the aforementioned methods is that their goal is to discover with as few tests as possible a falsifying (or a most likely falsifying) behavior of the SUT \cite{ernst2020arch}. 
In fact, SBTG methods have already established that they can falsify benchmark problems at a fraction of the cost of Monte-Carlo sampling, or even when Monte-Carlo sampling fails, e.g.,~\cite{ZhangEtAl2018cadics,abbas2013probabilistic}.
However, current SBTG methods (with exception maybe of~\cite{AbbasHFU14cyber,Fan2020}) cannot answer an important open question: what are the conclusions to be drawn if no falsifying behavior has been discovered?
In other words, when the test/simulation budget is exhausted and no violation has been discovered, can we conclude that the SUT is safe, or that at least it is likely safe?
This is a challenging problem for almost all black-box SBTG methods since foundationally, the speedup in falsification detection is a result of smart sampling in the search space.  
This is a challenging problem even for SBTG approaches, e.g.,~\cite{NahhalD07cav,DreossiDDKJD15nfm}, which are driven by coverage metrics over the search space.
Nevertheless, a positive answer to this question is necessary if SBTG methods were ever to be adopted and incorporated into assurance procedures~\cite{HeimdahlEtl2016faa}.

In this paper, we develop a general framework that can assess the probability that falsification regions exist in the search space of the SUT.
Our working assumption is that the SUT can be represented as an input-output function $\mathcal{M} : \mathbb{R}^d \rightarrow   ( {\mathbb{R}_{\geq 0}}  \rightarrow \mathbb{R}^m) $, i.e., a function that takes as input a vector $\mbf{x} \in \mathbb{R}^d$ and returns as output a vector signal $\mbf{z} : {\mathbb{R}_{\geq 0}}  \rightarrow \mathbb{R}^m$ representing the (deterministic) behavior of the system. 
The vector $\mbf{x}$ can represent system initial conditions $x_0$, static parameters $p$, and/or any finite representation of an input signal $u$ to the SUT. 
The vector $\mbf{x}$ is usually sampled from a convex bounded space $\Designspace \subseteq \mathbb{R}^d$, i.e., $\mbf{x} \in\Designspace$, which in our case is a hypercube. 
These assumptions are standard in the literature and $\mathcal{M}$ can represent a software- or hardware-in-the-loop SUT, or a simulation model of the SUT. 
In addition, we assume that the system is checked for correctness against a formal requirement $\varphi$ which has a quantitative interpretation of satisfaction usually referred to as specification robustness $\rho_\varphi : ( {\mathbb{R}_{\geq 0}}  \rightarrow \mathbb{R}^m) \rightarrow \mathbb{R}\cup \{-\infty,\infty\}$, e.g., as in the case of Signal Temporal Logic (STL) and its variants~\cite{fainekos2009robustness,DonzeM10formats,AkazakiH15cav,BartocciEtAl2018survey}.
The robustness $\rho_\varphi$ is positive when the trajectory  $\mbf{z}$ satisfies the requirement and negative otherwise.
Moreover, the magnitude of  $\rho_\varphi$ represents how robustly the trajectory  $\mbf{z}$ satisfies (or not) the requirement. 
Using the robustness function, the falsification problem for CPS can be converted into a search problem where the goal is to find points that belong to the zero level-set of the function $f\left(\mbf{x}\right)= \rho_\varphi(\mathcal{M}(\mbf{x}))$ for $\mbf{x}\in\Designspace$.

In order to identify the zero level-set of $f$ and, hence, estimate the probability of falsification, we propose a family of stochastic algorithms, referred to as Part-X (Partitioning with X-distributed sampling). 
Part-X 
adaptively partitions the search space $\Designspace$ in order to enclose the falsifying points.
The algorithm uses local Gaussian process estimates in order to adaptively branch and sample within the input space. 
The partitioning approach not only helps us identify the zero level-set, but also to circumvent issues that rise due to the fact that the function $f$ is discontinuous. 
In fact, the only assumption we need on $f$ is that it is a locally continuous function.
In order to evaluate our approach, we have built an SBTG library in Python (to be publicly released after publication) that can use the RTAMT~\cite{NickovicY2020atva} or TLTk~\cite{CralleySHF2020rv} temporal logic robustness Python libraries.
We demonstrate our framework on selected functions from the optimization literature and on the F16 benchmark problem ~\cite{heidlauf2018verification}.
An additional feature of our framework is that Part-X can also be utilized for evaluating  test vectors generated by other SBTG tools.
Therefore, Part-X can also function as an evaluation tool for other falsification algorithms in assurance procedures, or even in competitions~\cite{ernst2020arch}.

Besides the aforementioned practical/applied contributions of our work, the Part-X framework also makes the following theoretical and technical contributions. 
First, it uses multiple local models as opposed to a single global process. 
This helps the algorithm to handle discontinuities, estimate the worst and best performing inputs, and identify disconnected zero-level sets.
Second, Part-X uses a branching criterion to reduce the number of regions generated.

\section{Literature Review}~\label{sec::literature}
This work spans two macro-areas: \begin{inparaenum}\item[(i)] Automatic falsification of cyber-physical systems; \item[(ii)] Learning level sets of non-linear non-convex black-box functions\end{inparaenum}. In the following, we first motivate and document the contribution and state of the art within the field of automatic falsification performed with optimization based approaches, and we identify the challenges (section~\ref{sec::stochfals}). Section~\ref{sec::statandOpt} reviews techniques that can support non-stationary learning problems, with focus on the learning of a level set. 

\subsection{Stochastic Optimization for falsification}\label{sec::stochfals}

Global algorithms hold asymptotic convergence guarantees to global optima, but \textit{often require a large number of evaluations to effectively reach a good solution}. Stochastic optimization techniques such as simulated annealing, genetic algorithms, ant colony optimization, and the cross-entropy method have been applied in this domain. It is important to highlight that optimization algorithms are stochastic when: \begin{inparaenum}\item[(i)] an observed objective function evaluation is subject to noise; or \item[(ii)] randomness is injected via the search procedure itself while the function results are noiseless~\cite{spall2005introduction}, or a mixture of the two\end{inparaenum}. However, these methods notably lack sample efficiency, partly due to the difficulty in setting the numerous hyperparameters for methods with memory, and the inability of exploiting information from previous iterations for memory free methods. 
As an example of memory free sampling, hit and run, which is a common implementation of uniform random sampling, epitomizes myopic search: locations iteratively evaluated have no impact upon subsequent sampling decisions. The benefit of these stochastic search techniques is their easy-to-derive guarantees in terms of coverage. 
On the other hand, local search and hill climbing techniques (which are deterministic with noiseless function evaluations) such as CMA-ES, simplex search, trust region search, response surface methodology, and gradient ascent/descent have increased sample efficiency. However, due to the notorious non-linearity of robustness landscapes in CPS falsification, these local techniques often get trapped in sub-optimal local regions as these searches lack explorative properties. Among recent \textit{local} optimization approaches, Stochastic Trust-Region Response-Surface Method (STRONG) iteratively executes two stages: (1) constructing local models through linear regression, when gradient information is not readily available, and (2) formulating and solving the associated trust-region subproblem~\cite{chang2013stochastic}. In~\cite{shashaani2015astro}, the proposed Derivative Free Adaptive Sampling Trust Region Optimization (ASTRO-DF) uses derivative free trust-region methods to generate and statistically certify local stochastic models, that are used to update candidate solution values through an adaptive sampling scheme. \textit{Without assumptions on the degree of non-linearity in the objective function, the drawback of these methods is that the quality of the discovered local minimum may be poor, relative to the entire surface}. Research has been conducted on multi-starts and restart policies which have overcome these issues in some cases~\cite{zabinsky2010stopping}. \textit{Global} algorithms aim to balance the trade-off between exploration and exploitation, and investigate un-sampled regions without the promise of improvement. Notable examples include: improving hit-and-run~\cite{Zabinsky1993} with polynomial expected function evaluations, and the Greedy Randomized Adaptive Search Procedure (GRASP)~\cite{feo1995greedy}. 
A drawback to these methods is when function calls are limited. In such cases, meta-model based search uses previous samples to predict function values in regions where sampling has not been performed. Efficient Global Optimization (EGO)~\cite{jones1998efficient}, for deterministic black-box settings, is an established meta-model based method. In this context, Bayesian optimization (BO) is a popular black-box stochastic optimization method~\cite{frazier2018bayesian}. BO balances exploration and exploitation via surrogate modeling to produce high quality solutions in a relatively small number of iterations. However, due to the overhead costs associated to BO, such as surrogate model estimation and acquisition function optimization, this technique should be employed when observations of the objective function are expensive to collect - as in the case of observing the robustness $\rho_\varphi (\mathcal{M}(\mathbf{x}))$ for a given input $\mathbf{x}$. BO has proven to be quite successful over CPS falsification problems~\cite{deshmukh2017testing,ghosh2018verifying,waga2020falsification}. Recently, BO was combined with a local trust region search in an intelligent global-local optimization framework and proved highly effective for CPS falsification~\cite{mathesen2021stochastic,mathesen2019falsification}. 
  

\subsection{Statistical Methods for iterative Partitioning for Function Learning and Level Set Estimation}\label{sec::statandOpt}
Partitioning has deep roots in exact optimization (branch and bound methods), in black box optimization (target level set estimation), and in statistical applications such as additive modeling. In fact, partitioning is used: \begin{inparaenum}\item[(i)] to handle large scale data sets, by iteratively splitting the data; \item[(ii)] to handle high dimensional inputs, by iteratively splitting the original space into lower dimensional subspaces; \item[(iii)] to handle discontinuities of the reward function with respect to the input space, by iteratively branching the support of the decision space forming a partition\end{inparaenum}. In this paper, we focus on the ability to estimate a surrogate for the robustness function that allows the needed flexibility to represent a function with potentially high rate of variation. Since our Part-X uses Gaussian processes (see section~\ref{sec::GPLit}), we are challenged by the correlation function which is assumed to be constant. Partitioning allows us to change the rate of variation, through the Gaussian process correlation function, across the input space. Furthermore, we are interested in the estimation of the, potentially disconnected, set of locations that violate desired system properties. In this brief review, we focus on methods that use partitioning for level set estimation, and we highlight, when present, the type of guarantees offered by the different algorithms.
\paragraph{Surrogate Driven Approaches for Level Set Estimation}
In~\cite{shekhar2019}, the authors propose and analyze an algorithm for the estimation of a level set for the case where the reward function can be evaluated with noise. The authors estimate a Gaussian process across the entire solution space and, at each iteration, each subregion within the current partition is further branched based on the estimated distance to the, unknown, level set, and based on the predicted variance at the centroid of the subregion. If a subregion is branched, a new location is also evaluated and more replications are ran for the subregions forming the updated partition thus updating the Gaussian process. While the scope of the paper is different from ours, the authors do not provide and error bound over the level set estimate, and use a unique model across the partition. Also, the sampling is accuracy driven and does not consider the maximization of the probability to identify the level set. Within the falsification literature,~\cite{Fan2020} proposes to use partitioning to identify falsifying level sets. Specifically, assuming a partitioning and sampling schemes are provided as input, sequentially uses Conformal regression to derive an estimate of the maximum and minimum function value within a subregion, with subregion-wide noise. Based on such predictions, a subregion can be eliminated, maintained (if we have confidence that the system satisfies the properties) or further branched. While conformal regression does not allow to produce a point estimate of the response, the method allows to derive a probabilistic guarantee over a subregion for the correctness of the eliminating and maintaining. The work in~\cite{Lei2018} is used to define the conditions over the conformal regression. Besides the lack of point estimates, another potential drawback is that the authors assume that the sampling density is defined as input. Nonetheless, sampling densities can impact the performance of the approach. Even if not using surrogate models, we mention Probabilistic Branch and Bound (PBnB) that uses a directed random search to approximate a target level set associated with a target quantile of the best globally optimal solutions~\cite{Pra05, Zabinsky2020PBnB}. An advantage of PBnB is that it partitions the space iteratively, and performs more function evaluations in the promising regions as it refines its approximation of the target level set. In addition, there is a probabilistic bound on the error of the approximation, which we will use in the analysis of our algorithm (section~\ref{sec::pXtheory}).
\paragraph{Surrogates for non-stationary responses}
As mentioned, one possible criticality in using one single surrogate model is in the difficulty to capture discontinuities in the function behavior. Within the statistical learning community, so-called \textit{treed Gaussian processes} have been proposed with main applications in learning from non-stationary and large data sets. This literature does not address level set estimation, but it is relevant in the sense that it addresses non-stationarity of the response and the noise. In~\cite{Chipman2002} a binary tree is used to learn partitions based on Bayesian regression. The difficulty to scale the approach to high dimensional inputs/ large data sets led to the computational work in~\cite{DENISON2002}. One drawback of these approaches was identified in the irergularity of the variance associated to the different subregions. In particular, it was observed that some subregions tended to exhibit variances orders of magnitude larger. In~\cite{Kim2005} this problem is alleviated by using stationary processes in the several subregions (thus leading to a larger number of subregions, but better control over the variance profile). Differently,~\cite{Gramacy05bayesiantreed} deals directly with the problem of non-stationarity of the response and of the variance (heteroscedasticity) proposing a new form of Gaussian process, the Bayesian Treed Gaussian Process model. The approach combines stationary Gaussian processes and partitioning, resulting in treed Gaussian processes, and it implements a tractable non-stationary model for non-parametric regression. Along a similar line, \cite{Liang2018} also uses a binary tree which to iteratively learn different models and while the method has good fitting results, it also shows good computational performance for large datasets.
\paragraph{Reinforcement learning approaches}
As previously mentioned, sampling can have an important impact on the ability of the search method to perform effective partitioning, i.e., branching decisions that can increase either the prediction accuracy or the probability to identify a falsifying input. In this regard, it is important to point at relevant references in the area of optimal sequential sampling. In~\cite{Wang2019}, the latent action Monte Carlo Tree Search (LA-MCTS) sequentially focuses the evaluation budget by iteratively splitting the input space to identify the most promising region. Differently from our case, LA-MCTS uses non-linear boundaries for branching regions, and, like us, learns a local model to select candidates. Another relevant contribution, AlphaX, presented in~\cite{Wang2019AlphaX}, explores the search using distributed Monte Carlo Tree Search (MCTS) coupled with a Meta-Deep Neural Network (DNN) model that guides the search and branching decision, focusing on the sequential selection of the most promising region. In general, while relevant in terms of sampling, sequential optimization methods are designed to learn and focus on \textit{the} elite region of the input space. Our premise is different: as long as a point is falsifying the user will be interested in having this information no matter how bad the falsification is. 

\subsection{Contributions}\label{sec::contrib}
In this paper, we propose Part-X, a partitioning algorithm that relies on local Gaussian process estimates in order to adaptively branch and sample within the input space. Part-X brings the following innovative features over PBnB and prior falsification approaches:
\begin{enumerate}
    \item Due to the local metamodels, as opposed to a unique global process, Part-X produces point estimates with associated predictions, for the output function. This allows to build estimates for the worst and best performing inputs, as well as the volume of falsifying sets for arbitrary levels of negative robustness.
    \item The presence of a branching criterion avoids proliferation of subregions as a subregion is only branched when a criterion is satisfied (e.g., we are not able to classify the region).
    \item Under mild assumptions, the algorithm asymptotically achieves minimum error as a function of the minimum allowed size for a partition. 
    \item Part-X mechanism to partition, classify and model can be used to support and evaluate any other test generation approach. This will be further detailed in section~\ref{sec::PartXalgo}.
\end{enumerate}

\section{Part-X: a family of algorithms for partitioning-driven Bayesian optimization}\label{sec::PartXalgo}
\begin{figure}[H]
		{\includegraphics[width=0.85\textwidth]{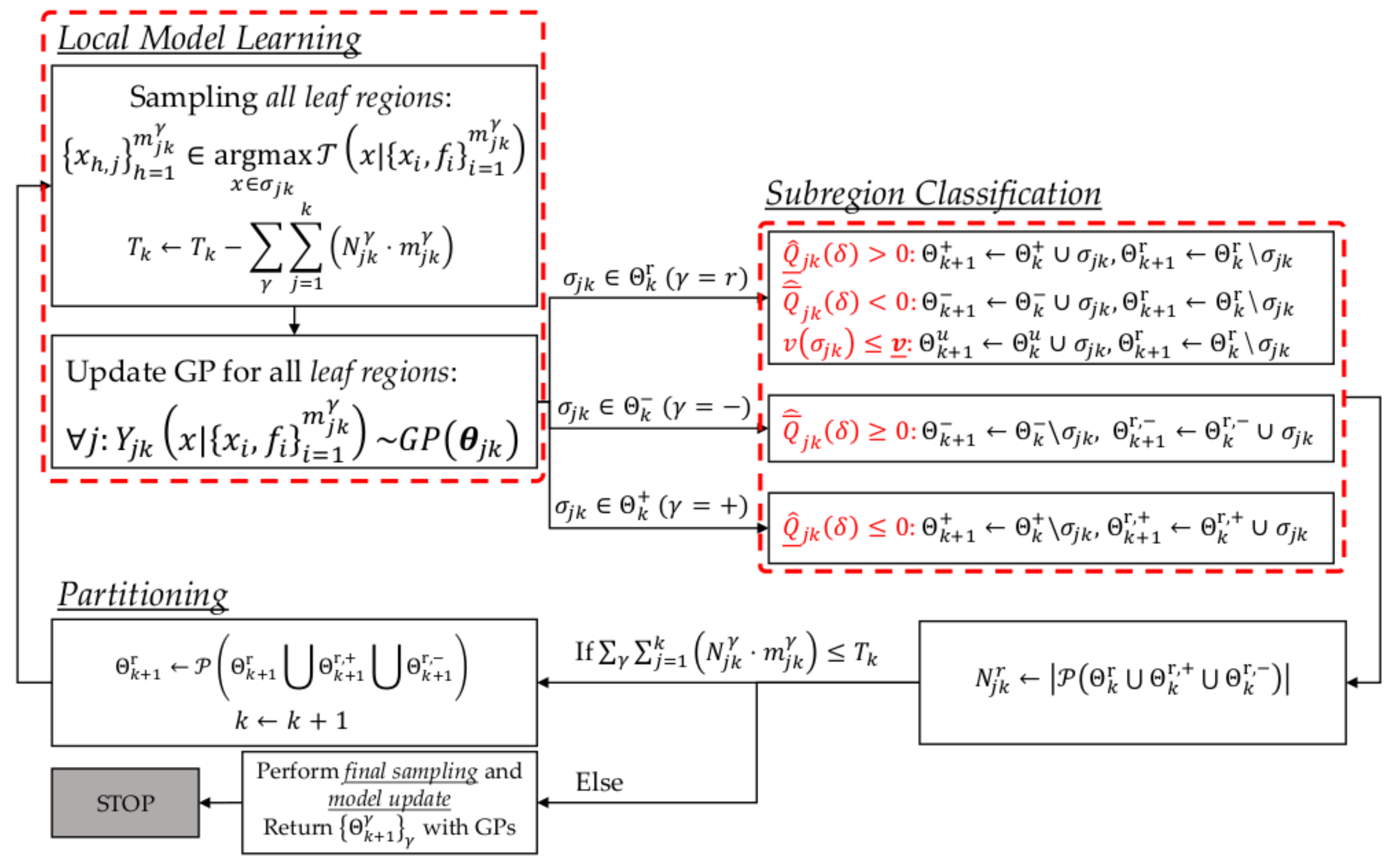}
		\caption{An overview of the Part-X algorithm.\label{fig::Part-X}}}
\end{figure}
Given the robustness function $f\left(\mbf{x}\right)= \rho_\varphi(\mathcal{M}(\mbf{x}))):\mathbb{R}^{d}\rightarrow \mathbb{R}, \mbf{x}\in\Designspace$, we aim at identifying the minimum for the function $f$, while producing an estimate of the function at locations $\mbf{x}$ that satisfy $f\left(\mbf{x}\right)\le 0$, and an estimate of the associated set of values $\LO$, such that $f\left(\mbf{x}\right)\le 0,\forall \mbf{x}\in \LO$. 
The algorithm we propose is an adaptive partitioning and sampling approach that provides: (i) the probability that a falsification exists within the input region $\Designspace$; (ii) the estimate of the falsification volume interpreted as $\FalsVol=v\left(\LO\right)/v\left(\Designspace\right)$, where $v\left(\cdot\right)$ represent the measure, volume in our application, associated to the relevant set. The probability in (i) is strongly dependent on the magnitude of the violation (i.e., how negative the robustness function is predicted to be), while the measure in (ii) helps engineers understanding how likely is the system to produce a falsifying behaviour (i.e., convergence of the estimated level set to the true level set w.p.1).

Figure~\ref{fig::Part-X} gives an overview of the proposed approach, which we refer to as Part-X (Partitioning with X-distributed sampling). 
The algorithm sequentially and adaptively partitions the space and evaluates inputs (test vectors) in order to estimate a surrogate for the robustness function in each of the subregions that are sequentially generated. 
In particular, the surrogates we use are Gaussian processes~\cite{santner2013design,mathesen2021stochastic,pedrielli2020extended}, which allows us to define a variety of sampling distributions ($T\left(\mbf{x}|\cdot\right)$ in Figure~\ref{fig::Part-X}). 
At each iteration, a number $N^{k}$ of points is sampled in each sub-region $\Subreg^{\gamma}_{ijk}$ to update the corresponding surrogate. Then, we decide whether to stop branching a region based on the fact that the posterior $\Significance$-quantile of the minimum (maximum) predicted robustness is above (below) the zero level, which deems an input to be unsafe. It is important to highlight the complexity associated with the estimation of the minimum and maximum $\Significance$-quantile associated to the Gaussian process $\min_{\mathbf{x}\in\Subreg^{\gamma}_{ijk}}\left[\gpPred\left(\mathbf{x}\right)-Z_{1-\Significance/2}\sqrt{\gpVar\left(\mathbf{x}\right)}\right]$, $\max_{\mathbf{x}\in\Subreg^{\gamma}_{ijk}}\left[\gpPred\left(\mathbf{x}\right)+Z_{1-\Significance/2}\sqrt{\gpVar\left(\mathbf{x}\right)}\right]$. In this work, we use a Monte-Carlo estimate for these quantities. 

At each iteration, a region can be \textit{temporarily classified}, thus entering the, potentially disconnected, set $\RegionPlus$ if the region is classified as \textit{not} falsifying, while $\RegionMinus$ is entered in the opposite scenario. In the case the uncertainty associated to the model(s) is large (which is typically the case at the first iterations of the algorithm) no sub-region is maintained, i.e., everything is branched. If a region reaches the minimum volume $v\left(\sigma\right)=\prod^{D}_{d=1} \minVol X_d$, where $X_d$ is the length of the robustness function support along dimension $d$, and $\minVol\in\left[0,1\right]$ is an input parameter. we cannot branch the subregion any more. The algorithm continues until: \begin{inparaenum}\item[(i)] the maximum number of evaluation is exhausted; \item[(ii)] all the subregions have been classified\end{inparaenum}.


Section~\ref{sec::GPLit} presents the basic definitions for Gaussian processes, which we use to produce predictions of the robustness function, section~\ref{sec::partitioning}, introduces the scheme followed by Part-X to iteratively branch, sample, update subregion models and decide whether to classify each of the subregions. 


\subsection{Modeling the Robustness as a Gaussian Process}\label{sec::GPLit}
A Gaussian process (GP) is a statistical learning model used to build predictions for non-linear, possibly non-convex smooth functions. The basic idea is to interpret the true, unknown function $y\left(\mathbf{x}\right)$ is a realization from a stochastic process, the Gaussian process. If we can measure the function without noise, then the Gaussian process will interpolate the true function at the evaluated points, while, conditional on the sampled locations $\mathbf{x}_{1},\ldots,\mathbf{x}_{n}$, a Gaussian process produces the conditional density $P\left(Y\left(\mathbf{x}_{0}\right)|\mathbf{x}\right)$. In particular, $Y(\mathbf{x}) = \mu + Z(\mathbf{x})$, where $\mu$ is the, constant, process mean, and $Z(\mathbf{x})\sim GP(0,\tau^2R)$, with $\tau^2$ being the constant process variance and $R$ the correlation matrix. Under the Gaussian correlation assumption, $R_{ij} = \prod_{l=1}^d \exp\left(-\theta_l \left(x_{il}-x_{jl}\right)^2\right)$, for $i, j, = 1,\ldots,n$. The $d$-dimensional vector of hyperparameters $\boldsymbol{\theta}$ controls the smoothing intensity of the predictor in the different dimensions. The parameters $\mu$ and $\tau^2$ are estimated through maximum likelihood~\cite{santner2013design}:
        $\hat{\mu} = \frac{\boldsymbol{1}_n^T\mathrm{R}^{-1}f(\boldsymbol{X}_n)}{\boldsymbol{1}^T_n\mathrm{R}^{-1}\boldsymbol{1}_n}, \ 
        \hat{\tau}^2= \frac{(f(\mathbf{X}_n)-\boldsymbol{1}_n\hat{\mu}_g)^T \mathrm{R}^{-1}((f(\mathbf{X}_n)-\boldsymbol{1}_n\hat{\mu}_g)}{n}$. The best linear unbiased predictor 
        form is~\cite{santner2013design}:
        \begin{eqnarray}\label{equ: gppre}
            \hat{f}(\mathbf{x}) = \hat{\mu} + \mathbf{r}^T\mathrm{R}^{-1}(f(\mathbf{X}_n) - \mathbf{1}_n\hat{\mu})\label{eqn::yhat}
        \end{eqnarray}
        where $\mathbf{X}_n$ is a set of $n$ sampled locations, and $f(\mathbf{X}_n)$ is the $n$-dimensional vector having as elements the function value at the sampled locations. The model variance associated to the predictor is:
        \begin{eqnarray} \label{equ: gpvar}
        s^2\left(\mathbf{x}\right) = \tau^2\left(1-\mathbf{r}^T\mathrm{R}^{-1}\mathbf{r} + \frac{\left((1-\mathbf{1}_n^T\mathrm{R}^{-1}\mathbf{r}\right)^2}{\mathbf{1}_n^T\mathrm{R}^{-1}\mathbf{1}_n} \right)\label{eqn::mvar}
        \end{eqnarray}
        where $\mathbf{r}$ is the $n$-dimensional vector having as elements the Gaussian correlation between location $\mathbf{x}\in\Designspace$ and the $n$ elements of $\mathbf{X}_n$, i.e., $\mathbf{r}_i(\mathbf{x}) =\prod_{l=1}^d \exp\left(-\theta_l(x_l-x_{il})^2\right), i = 1,\ldots,n$.

In our application, we use the model in~\eqref{eqn::yhat} as a surrogate for the, unknown, robustness function. In particular, given a training set of input and associated robustness value 
$\left\lbrace \mathbf{x}_i, y_i\right\rbrace^{n}_{i=1}$, we will 
predict the robustness value $\hat{Y}\left(\mathbf{x}_{n+1}\right)$ at a new unsampled location $\mathbf{x}_{n+1}$. The robustness prediction will have the associated variance $s\left(\mathbf{x}_{n+1}\right)$. 
%
%

\subsection{Sequential adaptive branching with classification}\label{sec::partitioning}

In this section, the model-based sequential adaptive partitioning is presented. 
Part-X starts considering the entire input $\Designspace$ and keeps branching until the simulation budget is exhausted (i.e., the maximum number of tests has been executed) or all the non-classified subregions have the a length $\delta$ along all the dimensions (i.e., the subregion is unbranchable), or all the subregions have been classified as either satisfying, violating or unbranchable. 

\begin{figure}[H]
\centering
		{\includegraphics[width=0.40\textwidth]{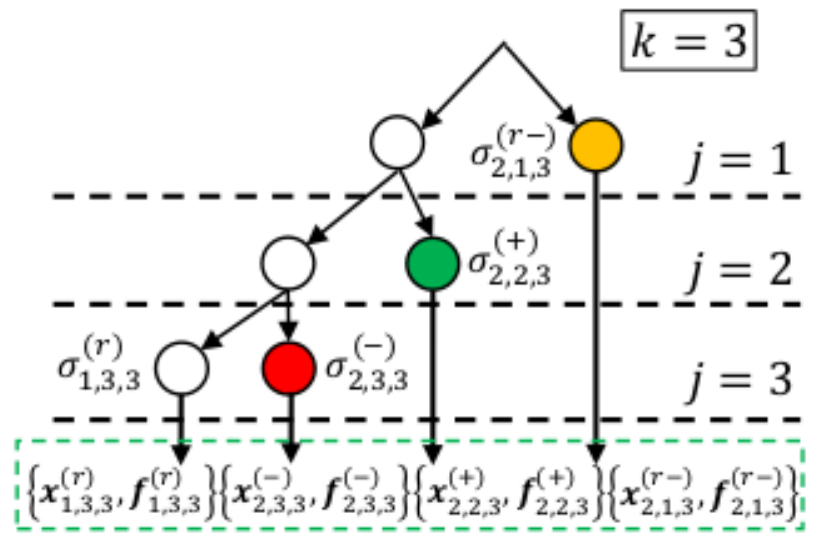}
		\caption{An example of partitioning tree generated by Part-X at iteration $k=3$. Subregions in red are classified as violating, those in green are satisfying, the orange region is being reclassified from violating to remaining (it would be light green if remaining from satisfying).\label{fig::Part-Xtree}}}
\end{figure}

Figure~\ref{fig::Part-Xtree} shows an example of the tree produced by the Part-X algorithm at iteration $k=3$. In general, at each iteration $k$, the Part-X tree be updated with the following subregion (leaf) types: 
\begin{itemize}
    \item New satisfying region ($\Subreg^{+}_{j,k}$): for each level $j=1,\ldots, k$, at each iteration $k$, the new set of satisfying regions, union of the new satisfying subregions, is $\Subreg^{+}_{j,k}=\bigcup_{i}\Subreg^{+}_{i,j,k}$, formed by new $N^{+}_{j,k}$ subregions. 
    \item New positive reclassified region ($\Subreg^{r+}_{j,k}$): for each level $j=1,\ldots, k$, at each iteration $k$, the new set of regions reclassified from positive, union of the new reclassified subregions, is $\Subreg^{r+}_{j,k}=\bigcup_{i}\Subreg^{r+}_{i,j,k}$, formed by new $N^{r+}_{j,k}$ subregions. 
    \item New violating region ($\Subreg^{-}_{j,k}$): for each level $j=1,\ldots, k$, at each iteration $k$, the new set of regions classified as violating, union of the new violating subregions, is $\Subreg^{-}_{j,k}=\bigcup_{i}\Subreg^{-}_{i,j,k}$, formed by new $N^{-}_{j,k}$ subregions.
    \item New negative reclassified region ($\Subreg^{r-}_{j,k}$): for each level $j=1,\ldots, k$, at each iteration $k$, the new set of regions reclassified from violating, union of the new reclassified subregions, is $\Subreg^{r-}_{j,k}=\bigcup_{i}\Subreg^{r-}_{i,j,k}$, formed by new $N^{r-}_{j,k}$ subregions.
    \item New remaining region ($\Subreg^{r}_{j,k}$): for each level $j=1,\ldots, k$, at each iteration $k$, the new set of regions remaining, union of the new remaining subregions, is $\Subreg^{r}_{j,k}=\bigcup_{i}\Subreg^{r}_{i,j,k}$, formed by new $N^{r}_{j,k}$ subregions.
    \item New unclassified region ($\Subreg^{u}_{j,k}$): for each level $j=1,\ldots, k$, at each iteration $k$, the new set of regions remaining, union of the new unclassified subregions, is $\Subreg^{u}_{j,k}=\bigcup_{i}\Subreg^{u}_{i,j,k}$, formed by new $N^{u}_{j,k}$ subregions.
\end{itemize}
We refer to $\Subreg^{\gamma}_{ijk}$ as the individual subregion of type $\gamma\in\left\lbrace +,-,r+,r-,r,u\right\rbrace$, resulting from branching at iteration $k$, at level $j$ of the partitioning tree, given the previous definitions, at the $k$\textsuperscript{th} iteration, there will be a number $N_{k}=\left(N^{+}_{k}+N^{r,+}+N^{-}_{k}+N^{r,-}+N^{r}_{k}+N^{u}_{k}\right)$ of \textit{new} subregions. We refer to the union of the positively classified subregions at level $j$, at iteration $k$ as $\RegionPlus_{jk},\RegionRp_{jk},\RegionMinus_{jk},\RegionRm_{j,k},\RegionUnd_{jk},\RegionR$ for satisfying, reclassified from satisfying, violating, reclassified from violating, undefined and remaining, respectively. Dropping the index $j$ will result in the notation for the same sets at each iteration.

\vspace{2pt}
\noindent{\textbf{Sampling and model estimation} } Part-X samples in a different way subregions that have been reclassified, or are not classified, i.e., $\gamma\in\left\lbrace r+,r-,r,u\right\rbrace$, and regions of type $\gamma\in\left\lbrace +,-\right\rbrace$). Each subregion in the first group requires at least $n_0$ points for the estimation of the Gaussian process (and associated predictor $\hat{Y}^{\gamma}_{ijk}$), and we also add $n_{\mbox{\tiny{BO}}}$ observations that are sampled using a Bayesian optimization approach, thus biasing sampling toward locations that falsify the requirements. In fact, the samples are sequentially collected in a way that maximizes the Expected Improvement~\cite{jones1998efficient}, and 
at each sampling iteration $t=1,2,\ldots,n_{\mbox{\tiny{BO}}-1}$, for each subregion $\Subreg^{\gamma}_{ijk}$, 
we sample a new location that maximizes the Expected Improvement $\mbox{EI}(\mathbf{x})$, namely: 
\begin{eqnarray}
    \mathbf{x}_{t+1}\in\arg\max_{\mathbf{x}\in\Subreg^{\gamma}_{ijk}}\mbox{EI}(\mathbf{x}) =  E\left[\max\left(\left[f^*-\hat{Y}^{\gamma}_{ijk}\left(\mathbf{x}\right)\right]\Phi\left(\frac{f^*-\hat{Y}^{\gamma}_{ijk}\left(\mathbf{x}\right)}{\hat{s}^{\gamma}_{ijk}\left(\mathbf{x}\right)}\right)+
\hat{s}^{\gamma}_{ijk}\left(\mathbf{x}\right)\phi\left(\frac{f^*-\hat{Y}^{\gamma}_{ijk}\left(\mathbf{x}\right)}{\hat{s}^{\gamma}_{ijk}\left(\mathbf{x}\right)}\right),0\right)\right].\label{eqn::eidef}
\end{eqnarray}
Where $f^*$ is the best function value sampled so far in subregion $\Subreg^{\gamma}_{ijk}$. Once the point has been sampled, we update the Gaussian process, and proceed until $n_{\mbox{\tiny{BO}}}$ evaluations have been performed. We then proceed verifying the branching conditions and possibly updating the partition. 

On the other hand, classified subregions receive an \textit{overall} evaluation budget of $n_{c}$ evaluations to be distributed across \textit{all} the subregions in this group. In order to perform such distribution of evaluations, we consider the Gaussian Process predictor $\hat{Y}^{\gamma}_{ijk}$, and we add samples to each subregion in this group using the following metric:
\begin{eqnarray}
    I^{\gamma}_{ijk}=\frac{1}{v\left(\Subreg^{\gamma}_{ijk}\right)}\int_{x_{0}\in \sigma^{\gamma}_{ijk}} \left(\int^{0}_{-\infty}f^{\gamma}_{ijk}\left(y\left(x_{0}\right)\right)dy\right)dx_{0}.\label{eqn::pm}
\end{eqnarray}
The basic idea behind the metric in~\eqref{eqn::pm} is to sample a region proportionally to the cumulated density below $0$ of the Gaussian process. The scalar $v\left(\Subreg^{\gamma}_{ijk}\right)$, representing the volume of the subregion $\Subreg^{\gamma}_{ijk}$, is used to normalize the indicator so that $I^{\gamma}_{ijk}\in\left(0,1\right)$. Note that, by construction $v\left(\Subreg^{\gamma}_{ijk}\right)>0$ and~\eqref{eqn::pm} exists finite.

An overview of the procedure for the sampling phase is reported in Algorithm~\ref{alg::SampleBO}. In general, we refer to $n_{jk}$ as the \textit{cumulated} number of evaluations in each subregion at level $j$ of the partitioning tree at iteration $k$.

\begin{algorithm}[htbp]
    \small{
       \caption{Sequential subregion sampling with Bayesian optimization (\texttt{SampleBO})}
       \label{alg::SampleBO}
    \begin{algorithmic}
       \STATE {\bfseries Input:} Subregion $\Subreg^\gamma_{ijk}\subset \mathbb{R}^d$, objective function $f(\mathbf{x})$, initialization budget $n^0$, total budget $n^{\mbox{\tiny{BO}}}$, $n_{jk}$ locations sampled so far $\left(\mbf{x}^\gamma_{ijk},\mbf{f}^\gamma_{ijk}\right)$;
       \STATE {\bfseries Output:} best location and value  $\mathbf{x}^{*}_{ijk}\in\Subreg^\gamma_{ijk}$, $f\left(\mathbf{x}^{*}_{ijk}\right)$, final Gaussian process model $\left(\gpPred(\mathbf{x}), \gpVar(\mathbf{x})\right)$.;
       \vspace{2pt}
       \hrule
       \vspace{2pt}
       \STATE \textbf{Step 1}: Compute the initial required evaluation budget:
       \IF{$n_{jk}\ge n_{0}$}
            \STATE Use $n_{jk}$ sampled points within the subregion as initializing points for the Gaussian process estimation; $t\leftarrow 0$;
       \ELSE
            \STATE Sample $n_{0}-n_{jk}$ points using a Latin Hypercube design. Return $\mathbf{x}_{\mbox{\scriptsize{train}}}\in\Subreg^\gamma_{ijk}$; $f\left(\mathbf{x}\right), \forall \mbf{x}\in \mathbf{x}_{\mbox{\scriptsize{train}}}$; $t\leftarrow n_0$;
       \ENDIF
       \WHILE{$t<n_{\mbox{\tiny{BO}}}$}
               \STATE \textbf{Step 2.1}: Estimate the GP using the training data $\{\mathbf{x}_{\mbox{\scriptsize{train}}},\mathbf{y}^i_{\mbox{\scriptsize{train}}}\}$, return $\left(\gpPred(\mathbf{x}), \gpVar(\mathbf{x})\right)$ for all $\mathbf{x}\in\Subreg^{\gamma}_{ijk}$;
               \STATE \textbf{Step 2.2}: Select the next location $\mathbf{x}^{*}_{\mbox{\scriptsize{EI}}} \leftarrow \arg \max_{\mathbf{x}\in \mathbb{X}} \mbox{EI}\left(\mathbf{x}\right)$; Evaluate and store $f(\mathbf{x}^{*}_{\mbox{\scriptsize{EI}}})$. 
                \STATE \textbf{Step 2.3}: $t \leftarrow t+1$\;
       \ENDWHILE
    \end{algorithmic}}
\end{algorithm}
\vspace{2pt}

\normalfont
\noindent{\textbf{Classification Scheme}} 
At the end of the sampling stage, we have $N$ Gaussian processes, and we need to estimate the $\Significance$-quantile for the minimum and maximum value of the function in each of the subregions. In order to do so, we use the Monte Carlo procedure in Algorithm~\ref{alg::MClim}.

\begin{algorithm}[H]
       \caption{Gaussian process based min-max quantiles estimation (\texttt{MCstep})}
       \label{alg::MClim}
    \small{
    \begin{algorithmic}
       \STATE {\bfseries Input:} subregion $\Subreg^\gamma_{ijk}\subset \mathbb{R}^d$, objective function $f(\mathbf{x})$, Gaussian process model $\left(\gpPred(\mathbf{x}), \gpVar(\mathbf{x})\right)$. Number of Monte Carlo iterations $R$, number of evaluations per iteration $M$;
       \STATE {\bfseries Output:} Estimates for the minimum and maximum $\Significance$-quantiles of the minimum and maximum function value $\widehat{\widebar{Q}}_{j}\left(\Significance\right),\Var\left(\widehat{\widebar{Q}}_{j}\left(\Significance\right)\right);\hat{\underbar{Q}}_{j}\left(\Significance\right), \Var\left(\hat{\underbar{Q}}_{j}\left(\Significance\right)\right)$ ;
       \vspace{2pt}
       \hrule
       \vspace{2pt}
       \FOR{$r=1,\ldots,R$}
           \FOR{$m=1,\ldots,M$}
               \STATE Sample $\mbf{x}_{mr}$, evaluate $\left(\gpPred(\mathbf{x}_{mr}), \gpVar(\mathbf{x}_{mr})\right)$;
           \ENDFOR
           \STATE Minimum and maximum quantile:
           \begin{eqnarray}
                \widebar{q}_r\left(\Significance\right) = \max\limits_{m=1,\ldots,M}\left(\gpPred\left(\mathbf{x}_{r,m}\right)+Z_{1-\frac{\Significance}{2}}\sqrt{\gpVar\left(\mathbf{x}_{r,m}\right)}\right),
                \underbar{q}_r\left(\Significance\right) = \min\limits_{m=1,\ldots,M}\left(\gpPred\left(\mathbf{x}_{r,m}\right)-Z_{1-\frac{\Significance}{2}}\sqrt{\gpVar\left(\mathbf{x}_{r,m}\right)}\right)\nonumber
            \end{eqnarray}
        \ENDFOR
        \STATE Minimum and maximum $\Significance$-quantile:
            \begin{eqnarray}
                \widehat{\widebar{Q}}_{j}\left(\Significance\right) = \frac{1}{R}\sum\limits_{i=1}^{R}\widebar{q}_r\left(\Significance\right),
                \Var\left(\widehat{\widebar{Q}}_{j}\left(\Significance\right)\right) = \frac{\Var\left(\widebar{q}_r\left(\Significance\right)\right)}{R},
                \hat{\underbar{Q}}_{j}\left(\Significance\right) = \frac{1}{R}\sum\limits_{i=1}^{R}\underbar{q}_r\left(\Significance\right),
                \Var\left(\hat{\underbar{Q}}_{j}\left(\Significance\right)\right) = \frac{\Var\left(\underbar{q}_r\left(\Significance\right)\right)}{R}\nonumber
            \end{eqnarray}
    \end{algorithmic}}
\end{algorithm}
\begin{algorithm}[htbp]
       \caption{Classification of a subregion (\texttt{Classify})}
       \label{alg::classify}
    \small{
    \begin{algorithmic}
       \STATE {\bfseries Input:} subregion $\Subreg^{\gamma}_{ijk}\subset \mathbb{R}^d$, current subregion type $\gamma\in\left(+,-,r\right)$, Gaussian process model $\left(\gpPred(\mathbf{x}), \gpVar(\mathbf{x})\right)$;
       \STATE {\bfseries Output:} Return region type $\left(r+,r-,+,-,r,u\right)$ ;
       \vspace{2pt}
       \hrule
       \vspace{2pt}
       \IF{$v\left(\Subreg_{j}\right)\le \prod^{D}_{d=1}\minVol\cdot X_d$}
            \STATE $\gamma = u$;
       \ELSIF{$\gamma=+$}
            \IF{$\widehat{\underbar{Q}}_{j}\left(\Significance\right)-Z_{1-\Significance/2}\Var\left(\widehat{\underbar{Q}}_{j}\left(\Significance\right)\right)\le 0$}
                \STATE $\gamma = r+$;
            \ELSE
                \STATE $\gamma = +$;
            \ENDIF
        \ELSIF{$\gamma=-$}
            \IF{$\widehat{\widebar{Q}}_{j}\left(\Significance\right)+Z_{1-\Significance/2}\Var\left(\widehat{\widebar{Q}}_{j}\left(\Significance\right)\right)\ge 0$}
                \STATE $\gamma = r-$;
            \ELSE
                \STATE $\gamma = -$;
            \ENDIF
        \ELSIF{$\gamma=r$}
            \IF{$\widehat{\widebar{Q}}_{j}\left(\Significance\right)+Z_{1-\Significance/2}\Var\left(\widehat{\widebar{Q}}_{j}\left(\Significance\right)\right)< 0$}
                \STATE $\gamma = -$;
            \ELSIF{$\widehat{\underbar{Q}}_{j}\left(\Significance\right)-Z_{1-\Significance/2}\Var\left(\widehat{\underbar{Q}}_{j}\left(\Significance\right)\right) > 0$}
                \STATE $\gamma = +$;
            \ELSE
                \STATE $\gamma = r$;
            \ENDIF
        \ENDIF
    \end{algorithmic}}
\end{algorithm}

Once the estimation procedure is complete, we classify each subregion $\Subreg^{r}_{ijk}$ according to the maximum and minimum quantile of the robustness function across the subregion. Specifically if the maximum quantile in a subregion satisfies $\widehat{\widebar{Q}}_{ijk}\left(\Significance\right)+Z_{1-\Significance/2}\Var\left(\widehat{\widebar{Q}}_{ijk}\left(\Significance\right)\right)<0$ the region is classified as violating the requirement and we update $\hat{\RegionMinus}_{k}\leftarrow \hat{\RegionMinus}_{k}\cup \Subreg^{r}_{ijk}$. On the other hand, if $\widehat{\widebar{Q}}_{ijk}\left(\Significance\right)-Z_{1-\Significance/2}\Var\left(\widehat{\widebar{Q}}_{ijk}\left(\Significance\right)\right)>0$ the region is classified as satisfying the requirement and we update $\hat{\RegionPlus}_{k}\leftarrow \hat{\RegionPlus}_{k}\cup \Subreg^{r}_{ijk}$. If a subregion is classified at iteration $k$, we verify if the corresponding classification criteria are violated and, in such scenario, re-classify the region as remaining $\Subreg^{r}_{ijk}$. 

Each time a subregion is classified as either satisfying or violating the requirements, the remaining region ($\RegionR$) is updated by removing the classified subregions. Given the remaining subregions $\hat{\RegionR}_{k}$, a branching algorithm is called that randomly selects a direction and cuts each subregion along that dimension into $B$ equal volume subregions. In particular, we allow subregions to be branched in direction $h=1,\ldots,d$ only if the size of the hypercube in that dimension is larger than $\minVol\times X_{d}$, where $\minVol$ is an input parameter, and  $X_{d}$ is the maximum length of the reward support along dimension $d$. Part-X terminates when the maximum number of function evaluations has been reached. We refer to this as the total evaluation budget $T$. 

\begin{algorithm}[H]
       \caption{Partitioning with Continued X-distributed Sampling}
       \label{alg::partX}
    \small{
    \begin{algorithmic}
       \STATE {\bfseries Input:} Input space $\Designspace$, function $f\left(\mbf{x}\right)$, initialization budget $n_0$, Bayesian optimization budget, $n_{\mbox{\tiny{BO}}}$, and unclassified subregions budget $n_{c}$, total number of evaluations $T$. Define branching operator $\mathcal{P}:A\rightarrow\left(A_{i}\right)_{i}:\bigcup_{i} A_{i}=A,\bigcap_{i} A_{i}=\emptyset$. Number of Monte Carlo iterations $R$, number of evaluations per iteration $M$; number of cuts per dimension per subregion $B$, classification percentile $\Significance$, $\minVol$;
       \STATE Set the iteration index $k\leftarrow 1$, Initialize the sets $\hat{\RegionMinus}_{k}=\hat{\RegionPlus}_{k}=\hat{\RegionRp}_{k}=\hat{\RegionRm}_{k}=\hat{\RegionUnd}_{k}=\emptyset$, $\RegionR_{k}\leftarrow\Designspace$;
       \STATE {\bfseries Output:} $\hat{\RegionMinus},\hat{\RegionPlus},\hat{\RegionRp},\hat{\RegionRm},\hat{\RegionR},\hat{\RegionUnd}$;
       \vspace{2pt}
       \hrule
       \vspace{2pt}
       \WHILE{$T_k\ge 0$}
        \STATE \underline{Branching}
            \FOR{$\Subreg^{\gamma}_{ijk}\in\left(\hat{\RegionRp}_{k}\cup \hat{\RegionRm}_{k}\cup \hat{\RegionR}_{k}\right)$}
                \STATE Return $\left(\Subreg^{r}_{i,j+1,k}\right)^{B}_{i=1}=\mathcal{P}\left(\Subreg^{\gamma}_{ijk}\right)$;
            \ENDFOR
            \STATE Count $N^{\widebar{C}}_{k}\leftarrow$ number of non-classified leaves of the partitioning tree at iteration $k$; $N^{C}_{k}\leftarrow$ number of classified leaves of the partitioning tree at iteration $k$;
            \IF{$T_{k}\ge n_{\mbox{\tiny{BO}}}\cdot N^{\widebar{C}}_{k} + \sum_{\Subreg^{\gamma}_{ijk}:\gamma\in\left\lbrace r+,r-,r\right\rbrace}\max\left(n_{jk}-n_0,0\right)$}
                \FOR{All unclassified subregions $\Subreg^{r}_{i,j,k}, \forall j, i(j)$}
                    \STATE Execute \texttt{SampleBO}$\left(\Subreg^{r}_{i,j,k},n_{\mbox{\tiny{BO}}},n_{0},n_{jk}\right)$;
                    \STATE Return the quantiles for the minimum and maximum function value executing \texttt{MCstep}$\left(R,M,\Subreg^{r}_{i,j,k},\gpPred,\gpVar\right)$;
                    \STATE Update subregions type: $\gamma\leftarrow$\texttt{Classify}$\left(\Subreg^{r}_{i,j,k},\gpPred,\gpVar\right),\hat{\Theta^{\gamma}}_{k}\leftarrow \hat{\Theta^{\gamma}}_{k}\bigcup \sigma_{j}, \hat{\Theta^{r}}_{k}\leftarrow\hat{\Theta^{r}}_{k}\setminus \sigma_{j}$;
                    \STATE $N^{\gamma}_{jk}\leftarrow N^{\gamma}_{jk}+1$;
                    \STATE $T_{k}\leftarrow T_{k}-n_{\mbox{\tiny{BO}}}-\max\left(n_{jk}-n_0,0\right)$;
                \ENDFOR
                \FOR{$j=1,\ldots,N^{+}_{k}\cup N^{-}_{k}$}
                    \STATE Allocate $n_c$ across the subregions proportionally to the probability metric in eqn.~\eqref{eqn::pm};
                    \STATE Return the quantiles for the minimum and maximum function value executing \texttt{MCstep}$\left(R,M,\Subreg^{\gamma}_{i,j,k},\gpPred,\gpVar\right)$;
                    \STATE Update subregions type: $\gamma\leftarrow$\texttt{Classify}$\left(\Subreg^{r}_{i,j,k},\gpPred,\gpVar\right),\hat{\Theta^{\gamma}}_{k}\leftarrow \hat{\Theta^{\gamma}}_{k}\bigcup \Subreg^{\gamma}_{ijk}, \hat{\Theta^{r}}_{k}\leftarrow\hat{\Theta^{r}}_{k}\setminus \Subreg^{\gamma}_{ijk}$;
                \ENDFOR
                \STATE $T_{k}\leftarrow T_{k}-n_c$;
            \ELSE
                \STATE Allocate $T_{k}$ to the subregions proportionally to the volume;
                \STATE Evaluate the function at the sampled point and update the Gaussian processes;
                \STATE $T_{k}\leftarrow 0$;
            \ENDIF
        \STATE $k\leftarrow k+1$;       
       \ENDWHILE
       \STATE Return $\FalsVol = \frac{v\left(\hat{\RegionMinus}_{k-1}\right)}{v\left(\Designspace\right)}$ 
    \end{algorithmic}}
\end{algorithm}

\vspace{2pt}
\noindent{\textbf{Partitioning with X-distributed sampling}}
The procedure in Algorithm~\ref{alg::partX} summarizes the phases of the proposed approach. In the algorithm, we use the notation $v\left(\cdot\right)$ to refer to the volume of a region. Since the regions in Part-X are hyperboxes, volumes are easily computable.

\begin{remark}[Part-X as evaluation tool]
 As mentioned in section~\ref{sec::contrib}, Part-X can be used by any test generation to evaluate it in its ability to identify falsifying regions. In order to do so Algorithm~\ref{alg::partX} can be executed without sampling. Doing so, the algorithm will keep branching until a subregion is classified, there are less than $n_{0}$ points in a subregion, the subregion has reached the  minimum volume. Doing so the falsification volume will be computed and different test generation tools can be compared.
\end{remark}

\section{Part-X Theoretical Analysis}\label{sec::pXtheory}
 
Figure~\ref{fig::thmQsv2} shows the 7\textsuperscript{th}, and 8\textsuperscript{th} iterations of Part-X applied to the 2-$d$ Himmelblau's function (section~\ref{sec::empirical}). The Figure~\ref{fig::thmQsv2} highlights the main quantitie at the core of our algorithm analysis, i.e., the concept of mis-classification volumes for the classified regions (for brevity referred to as $\Delta^{+}_{k}$ and $\Delta^{-}_{k}$ in Figure~\ref{fig::thmQsv2} for the satisfying and violating region, respectively), and mis-classification events $C^{+}_{k}$, and $C^{-}_{k}$. The following results provide bounds, with associated guarantees, of the 
mis-classification volume \textit{at each iteration of the algorithm}. We start providing an important result from the literature that allows us to analyse the error within a single subregion. Based on this result, we bound misclassification error and recovery of such error, both for satisfying as well violating subregions. Such bound is probabilistic in nature and it is defined at each level of the partitioning tree, and for each iteration of the algorithm (Lemmas~\ref{lem::misclassifyRp}-\ref{lem::misclassifym}). We then extend these result to the overall error in Theorems~\ref{thm:seq_misclassifyPlusCondC}-\ref{thm:seq_misclassifyMinusCondC}. 

\begin{figure}[h!]
\centering
		{\includegraphics[width=0.45\textwidth]{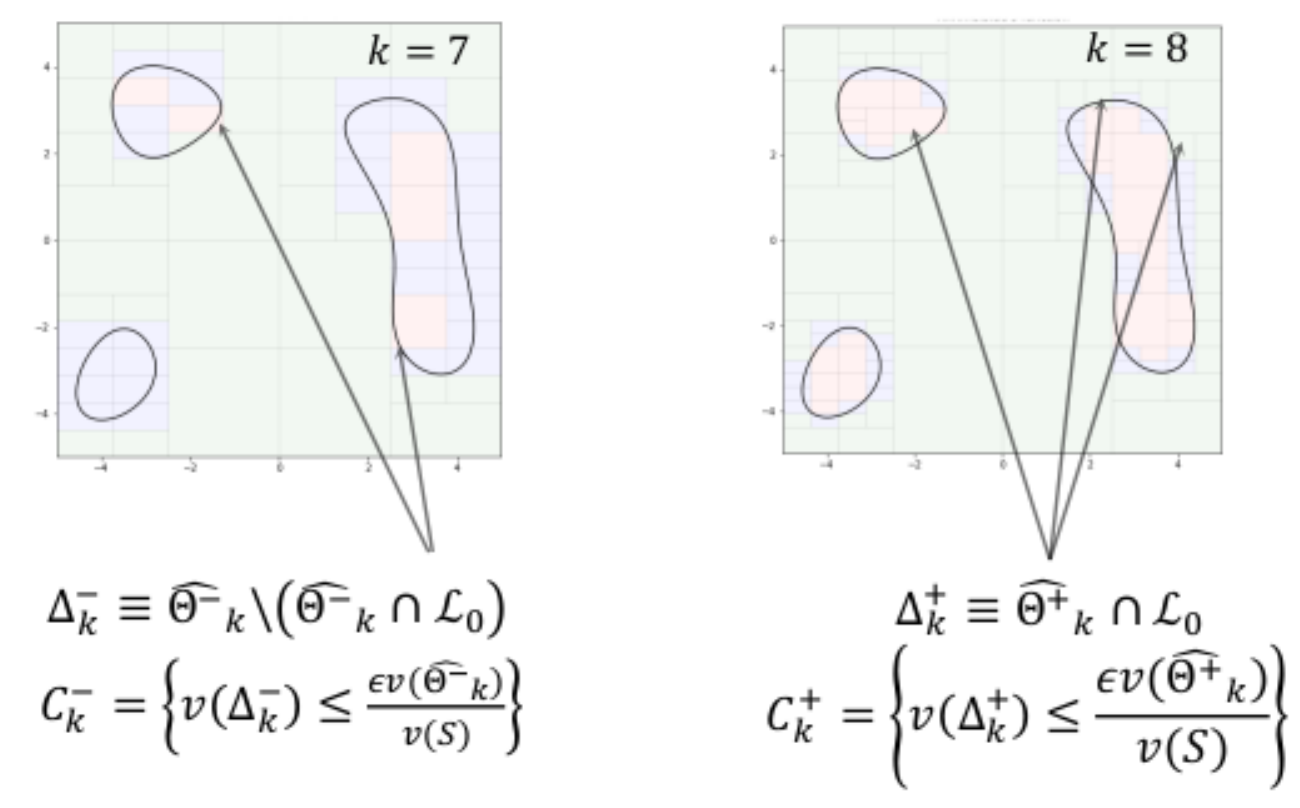}
		\caption{An example of two consecutive iterations for the Part-X algorithm. On the left we see a mis-classification event due to the fact that two subregions classified as violating, are in fact containing satisfying locations. On the right, the opposite scenario is shown. While we cannot, even asymptotically, eliminate this error, we will give probabilistic guarantees on its associated volume measure at each iteration of Part-X.\label{fig::thmQsv2}}}
		\vspace{-9pt}
\end{figure}

\subsection{Assumptions and Notation}~\label{Sec::assunot}
\begin{assumption}\label{ass::compactX}
	$\Designspace$ is a compact space.
\end{assumption}

\begin{assumption}\label{ass::locsmooth}
	$f\left(\mbf{x}\right)$ is locally smooth, i.e., there exists a collection of subregions with positive Lebesgue measure such that $f\left(\mbf{x}\right)$ is smooth within each subregion.
\end{assumption}

\begin{assumption} \label{known params}
	The hyperparameters $\tau$, $\boldsymbol{\theta}$ of the Gaussian process model are assumed to be known.
\end{assumption}

\begin{assumption}\label{cross valid}
	The initial sample set of $n_{0}$ points $\{\boldsymbol{x}_{i}\}_{i=1}^{n_0}$, produces a Gaussian process model that satisfies cross-validation.
\end{assumption}

\begin{assumption}\label{bounded function}
	The true function to be optimized, $f$, is bounded over $\Designspace$.
\end{assumption}

\begin{definition}[Significance]
\label{def::sign}
Let $\alpha_{j}$, be the significance of a probabilistic statement at level $j$ of the partitioning tree. Then, given any $\alpha>0$, we will have that the significance at level $j$ satisfies:
\begin{eqnarray}
\alpha_{j}=\begin{cases}
\alpha & \mbox{If }j=1\\
\alpha_{j-1}/B & \mbox{If }j>1
\end{cases}\label{eqn::signStatements}
\end{eqnarray}
\end{definition}

\begin{definition}
We will refer to the $0$-level set $\LO$ as
\begin{eqnarray}
    \LO=\{\mathbf{x}\in S: f(\mathbf{x})\le 0\}\label{eqn::L0}
\end{eqnarray}
\end{definition}

\begin{definition}[$\epsilon^{+}_{j,k},\epsilon^{+}_{k},\widehat{\RegionPlus}_{j,k},\widehat{\RegionPlus}_{k}$]\label{def::epsp}
Let $\epsilon^{+}_{j,k} = v(\widehat{\RegionPlus}_{j,k} \cap \LO), j=1,\ldots,k$ denote the volume incorrectly classified as $\RegionPlus$ at the $j$\textsuperscript{th} level of the partitioning tree, up to the $k$\textsuperscript{th} iteration of the algorithm. Then, $\epsilon^{+}_{k} = \sum^{k}_{j=1}v(\widehat{\RegionPlus}_{j,k} \cap \LO)$ will be the volume incorrectly classified as $\RegionPlus$ up to the $k$\textsuperscript{th} iteration of the algorithm. Equivalently, $\epsilon^{+}_{k} = v(\widehat{\RegionPlus}_{k} \cap \LO)$, where $\widehat{\RegionPlus}_{k}=\bigcup^{k}_{j=1}\widehat{\RegionPlus}_{jk}$.
\end{definition}
\begin{definition}[$\epsilon^{-}_{j,k},\epsilon^{-}_{k},\widehat{\RegionMinus}_{j,k},\widehat{\RegionMinus}_{k}$]\label{def::epsm}
Let $\epsilon^{-}_{j,k} = v(\widehat{\RegionMinus}_{jk}) - v(\widehat{\RegionMinus}_{jk} \cap \LO), j=1,\ldots,k$ denote the volume incorrectly classified as $\RegionMinus$ at the $j$\textsuperscript{th} level of the partitioning tree, up to the $k$\textsuperscript{th} iteration of the algorithm. Then, $\epsilon^{-}_{k} = \sum^{k}_{j=1}\left[v(\widehat{\RegionMinus}_{jk}) - v(\widehat{\RegionMinus}_{jk} \cap \LO)\right]$ will be the volume incorrectly classified as $\RegionMinus$ up to the $k$\textsuperscript{th} iteration of the algorithm. Equivalently, $\epsilon^{-}_{k} = v(\widehat{\RegionMinus}_{k})-v(\RegionMinus_{k} \cap \LO)$, where $\widehat{\RegionMinus}_{k}=\bigcup^{k}_{j=1}\widehat{\RegionMinus}_{jk}$.  
\end{definition}
In the first part of the proof, we will assume that (level formulations will also be used applying definitions~\ref{def::epsp}-\ref{def::epsm}).
\begin{eqnarray} 
0 \le \epsilon^{+}_{k} \le \frac{\epsilon v\left(\widehat{\RegionPlus}_k\right)}{v\left(S\right)} ,
0 \le \epsilon^{-}_k \le \frac{\epsilon v\left(\widehat{\RegionMinus}_k\right)}{v\left(s\right)}. \nonumber
\end{eqnarray}
Where the user-defined parameter $\epsilon$ represents, for each subregion, the volume that the user tolerates to be wrongly classified. 
\begin{definition}[Misclassification events]\label{def::misclassE}
Let $C_k = (C_k^+\cap C_k^-)$ denote the mis-classification event at the $k$\textsuperscript{th} iteration to be defined as: 
\begin{eqnarray} 
\begin{aligned}
C^{+}_k &=\left\lbrace v\left(\widehat{\RegionPlus}_k \cap \LO\right) \le \frac{\epsilon v\left(\widehat{\RegionPlus}_k\right)}{v\left(s\right)}\right\rbrace. \label{eqn::cplus}\\
C^{-}_k &=\left\lbrace  v\left(\widehat{\RegionMinus}_k \cap \LO\right) \ge v\left(\widehat{\RegionMinus}_k\right)-\frac{\epsilon v(\widehat{\RegionMinus}_k)}{v\left(s\right)}\right\rbrace.
\end{aligned}
\end{eqnarray}

let event $C_k = \left(C_k^+\cap C_k^-\right)$ to ensure that the volume of incorrectly classified subregions satisfy the upper bound.

\end{definition}

Part-X returns $\hat{\RegionMinus}_{k}$, as an estimate for the true, unknown $0$-level set $\LO=\{\mathbf{x}\in S: f(\mathbf{x})\le 0\}$.

\subsection{Main Results}\label{sec::main_result}
The first result we provide exploits the PAC guarantee in~\cite{mcallester1999some} (Thm. 1, p. 358), and reformulates it for Part-X when the focus is the classification of the $i$\textsuperscript{th} subregion, at the $j$\textsuperscript{th} level of the tree, at the $k$\textsuperscript{th} iteration of the algorithm. Consider Algorithm~\ref{alg::classify} that classifies each subregion $\hat{\sigma}^{\cdot}_{ijk}$, then, the following can be proved.
\begin{lemma}[General PAC guarantee for subregion classification]\label{lem::PACgen}
Let us refer to $\mathcal{H}$ as the classification of the subregion $\hat{\sigma}^{\mathcal{H}}_{ijk}$. Then, we have for $\alpha_{j}>0$, that with probability at least $1-\alpha_{j}$ over the choice of a sample of $n_{jk}$ observations, the subregion $\hat{\sigma}^{\mathcal{H}}_{ijk}$, which is such that every element of $\hat{\sigma}^{\mathcal{H}}_{ijk}$ is consistent with the sample and the subregion is measurable ($P(\hat{\sigma}^{\mathcal{H}}_{ijk}) > 0$), satisfies the following.
\begin{eqnarray}
\varepsilon^{\mathcal{H}}_{ijk}\le \frac{\ln\frac{1}{P(\hat{\sigma}^{\mathcal{H}}_{ijk})}+\ln\frac{1}{\alpha_{j}}+2\ln n_{jk}+1}{n_{jk}}.\nonumber
\end{eqnarray}
where, the error rate $\varepsilon^{\mathcal{H}}_{ijk}$ associated to $\hat{\sigma}^{\mathcal{H}}_{ijk}$ is defined for a given desired hypothesis (target) $\mathcal{H}$ to be the probability over the choice of the input $x$ over $\hat{\sigma}^{\mathcal{H}}_{ijk}$ disagrees with $\mathcal{H}$.
\end{lemma}
\begin{proof}
This Lemma is a readaptation of the result in Thm. 1, p. 358 of~\cite{mcallester1999some}, therefore we omit the proof here.
\end{proof}

Lemma~\ref{lem::misclassifyP} and \ref{lem::misclassifym} bound the classification error for the positively and negatively classified volume, respectively, at each iteration $k$ of the Part-X algorithm. Lemma~\ref{lem::misclassifyRp} and \ref{lem::misclassifyRm} produce a bound on the reclassified error volume at each level $j$ of the partitioning tree, at iteration $k$ of the algorithm, conditional upon the event $C_k$. This volume is important because it is responsible for a decrease in the classification error due to the subregions wrongly classified as safe in previous iterations. Reclassification cannot occur at the first iteration (i.e., $k>1$) and the recovered error volume at the $k$\textsuperscript{th} is clearly bounded by the error at the previous iteration $k-1$.

\begin{lemma}[Positive classified error bound]\label{lem::misclassifyP}
Let $\LO$ denote the $0$-level set defined in~\eqref{eqn::L0}, and $\widehat{\Subreg}^{+}_{jk}$ be the set of subregions classified as satisfying at the $j$\textsuperscript{th} level of the Part-X tree, at the $k$\textsuperscript{th} iteration of the algorithm. Refer to $\widehat{\Subreg}^{+}_{k}$ as the set of subregions classified as satisfying at the $k$\textsuperscript{th} iteration of the algorithm. Assume the event $C_k$ in Definition~\ref{def::misclassE} to be true. Then, the following holds for the volume misclassified as satisfying at the $k$\textsuperscript{th} iteration of Part-X:  
\begin{eqnarray} 
P\left(v\left(\widehat{\Subreg}^{+}_{jk} \cap \LO \right)\le N^{+}_{jk}\epsilon^{+}_{jk} | C_k\right)\ge \prod^{N^{+}_{j,k}}_{i=1}\left[1-\delta_{ijk}\right]. \label{eqn::volspjk}
\end{eqnarray}
where, under event $C^{+}_{k}$, $0 \le \epsilon^{+}_{jk} \le \frac{\epsilon v\left(\widehat{\RegionPlus}_{jk}\right)}{v\left(S\right)}$, and $\delta_{ijk}=\frac{\ln\frac{1}{p_{ijk}}+\ln\frac{1}{\alpha_{jk}}+2\ln n_{jk}+1}{n_{jk}}$.
By construction:
\begin{eqnarray} 
P\left(v\left(\widehat{\Subreg}^{+}_{k} \cap \LO \right)\le N^{+}_{k}\epsilon_{k} | C_k\right)\ge \prod^{k}_{j=1}\prod^{N^{+}_{j,k}}_{i=1}\left[1-\delta_{ijk}\right].\label{eqn::volspk}
\end{eqnarray}
where, under event $C^{+}_{k}$, $0 \le \epsilon^{+}_{k} \le \frac{\epsilon v\left(\widehat{\RegionPlus}_{k}\right)}{v\left(S\right)}$.
\end{lemma}

\begin{proof}
By the definition of $\LO$, 
\begin{eqnarray} 
P\left(v\left(\widehat{\Subreg}^{+}_{j,k} \cap \LO \right)\le N^{+}_{j,k}\epsilon^{+}_{k}| C_{k}\right) = P\left(v\left(\mathbf{x}_k: f\left(\mathbf{x}_{j,k}\right)\le 0, \mathbf{x}_{j,k}\in \widehat{\Subreg}^{+}_{j,k}  \right)\le N^{+}_{j,k}\epsilon^{+}_{k}| C_{k}\right), \forall k.\label{eqn::lemma_volpsk1}
\end{eqnarray}
We assume that there exists $y_{j,k}^{+}$, such that
\begin{eqnarray} 
P\left(v\left(\mathbf{x}_{j,k}: f\left(\mathbf{x}_{j,k}\right)\le y_{j,k}^{+}, \mathbf{x}_{j,k}\in \widehat{\Subreg}^{+}_{j,k}  \right)\right) = \frac{v\left(\mathbf{x}_{j,k}: f\left(\mathbf{x}_{j,k}\right)\le  y_{j,k}^{+}\right)}{v\left(\widehat{\Subreg}^{+}_{j,k} \right)} = \frac{N^{+}_{j,k}\epsilon^{+}_{k}}{v\left(\widehat{\Subreg}^{+}_{j,k} \right)}.\label{eqn::lemma_volpsk2}
\end{eqnarray}
Therefore, by applying equation \eqref{eqn::lemma_volpsk2} to equation \eqref{eqn::lemma_volpsk1}, we have 
\begin{eqnarray}
P\left(v\left(\widehat{\Subreg}^{+}_{j,k} \cap \LO \right)\le N^{+}_{j,k}\epsilon^{+}_{k}| C_{k}\right) =
P\left(v\left(\mathbf{x}_{j,k}: f\left(\mathbf{x}_{j,k}\right)\le 0  \right)\le v\left(\mathbf{x}_{j,k}: f\left(\mathbf{x}_{j,k}\right)\le  y_{j,k}^{+}| C_{k}\right)\right).\nonumber
\end{eqnarray}
which is equal to $P\left(y_{j,k}^{+}<0| C_{k-1}\right)$. Part-X decides the characterization of each subregion based on the minimum quantile. Let us define the event $F_{ijk}$ occurring when $v\left(\widehat{\Subreg}^{+}_{i,j,k} \cap \LO \right)\le \epsilon^{+}_{i,j,k}$, such that:
\begin{eqnarray}
0\le\epsilon^{+}_{i,j,k}\le\frac{\epsilon v\left(\widehat{\Subreg}^{+}_{ij,k}\right)}{v\left(S\right)}.\nonumber
\end{eqnarray}
Consider as the PAC sampling measure, i.e., $P\left(U\right)$ in Lemma~\ref{lem::PACgen}, the $P\left(f\le 0|\mathcal{S}_{ijk}\right)$, where $\mathcal{S}_{ijk}$ is the set of sampled points in the $i$\textsuperscript{th} subregion at level $j$, at iteration $k$. Such probability is implied by our separable kernel in equations~\eqref{eqn::yhat}-\eqref{eqn::mvar}, and it is therefore known at each iteration. Since each subregion has associated a different Gaussian process, we refer to this probability as $p_{ijk}$. Then the probability associated to the event $F_{ijk}$ is:
\begin{eqnarray}
P\left(F_{ijk}\right)=P\left(v\left(\widehat{\Subreg}^{+}_{i,j,k} \cap \LO \right)\le \epsilon^{+}_{i,j,k}| C_{k}\right). \nonumber
\end{eqnarray}
where $\widehat{\Subreg}^{+}_{i,j,k}$ is the $i$\textsuperscript{th} subregion at level $j$ at iteration $k$. Let us now refer to the mis-classification probability of this positively classified volume as $\varepsilon^{+}_{ijk}$. Then, 
from Lemma~\ref{lem::PACgen}, we know that:
\begin{eqnarray}
P\left(\varepsilon^{+}_{ijk}\le \frac{\ln\frac{1}{p_{ijk}}+\ln\frac{1}{\alpha_{jk}}+2\ln n_{jk}+1}{n_{jk}}\right)\ge 1-\alpha_{jk}.\nonumber
\end{eqnarray}
It is important to connect the classified volume and the error rate $\varepsilon^{+}_{ijk}$. In particular, we have that:
\begin{eqnarray}
\varepsilon^{+}_{ijk} = \left[v\left(\widehat{\Subreg}^{+}_{i,j,k} \cap \LO \right)/v\left(S\right)\right].\label{eqn::relvrpvarepsPOS}
\end{eqnarray}
And, conditional on the event $C_{k}$, the relationship~\eqref{eqn::relvrpvarepsPOS} implies the following:
\begin{eqnarray}
\varepsilon^{+}_{ijk} \le \frac{\epsilon^{+}_{ijk}}{v\left(S\right)}.\nonumber
\end{eqnarray}
Given the definition of our event $F_{ijk}$, 
the following holds:
\begin{eqnarray}
P\left(v\left(\widehat{\Subreg}^{+}_{i,j,k} \cap \LO \right)\le \epsilon^{+}_{i,j,k}| C_{k}\right)\ge 1-\varepsilon^{+}_{ijk} \ge 1-\frac{\ln\frac{1}{p_{ijk}}+\ln\frac{1}{\alpha_{jk}}+2\ln n_{jk}+1}{n_{jk}}.\nonumber
\end{eqnarray}
Now, we want to derive the \textit{level-probability} as the measure associated to the event $F_{jk}=\bigcap_{i}F_{ijk}$. Since the $N^{+}_{j,k}$ subregions that get positively classified are independent, the following holds:
\begin{eqnarray}
P\left(F_{jk}|C_{k}\right)\ge \prod_{i}\left[1-\frac{\ln\frac{1}{p_{ijk}}+\ln\frac{1}{\alpha_{jk}}+2\ln n_{jk}+1}{n_{jk}}\right].\nonumber
\end{eqnarray}
Finally, we need to derive a bound for our original probability, i.e., $P\left(v\left(\widehat{\Subreg}^{+}_{j,k} \cap \LO \right)\le N^{+}_{j,k}\epsilon^{+}_{j,k}| C_{k}\right)$, clearly the following holds:
\begin{eqnarray}
P\left(v\left(\widehat{\Subreg}^{+}_{j,k} \cap \LO \right)\le N^{+}_{j,k}\epsilon^{+}_{j,k}| C_{k}\right)\ge P\left(F_{jk}|C_{k}\right).\nonumber
\end{eqnarray}
Hence:
\begin{eqnarray}
P\left(v\left(\widehat{\Subreg}^{+}_{j,k} \cap \LO \right)\le N^{+}_{j,k}\epsilon^{+}_{j,k}| C_{k}\right)\ge \prod_{i}\left[1-\frac{\ln\frac{1}{p_{ijk}}+\ln\frac{1}{\alpha_{jk}}+2\ln n_{jk}+1}{n_{jk}}\right]
=\prod^{N^{+}_{j,k}}_{i=1}\left[1-\delta_{ijk}\right].\nonumber
\end{eqnarray}
Now, we need to derive the iteration-level guarantee. We notice that, for different levels $1<j\le k$, the sampling and classification are independent processes. In fact, as the volume $v\left(\widehat{\Subreg}^{+}_{jk} \cap \LO \right)$ is the union of the disjoint regions $v\left(\widehat{\Subreg}^{+}_{ijk} \cap \LO \right)$, so it is the volume $v\left(\widehat{\Subreg}^{+}_{k} \cap \LO \right)$ resulting from the union of the disjoint volumes $v\left(\widehat{\Subreg}^{+}_{jk} \cap \LO \right), j\le k$. 

Let us define the event $A_{jk}$ to occur when $v\left(\widehat{\Subreg}^{+}_{j,k} \cap \LO \right)\ge N^{+}_{j,k}\epsilon^{+}_{j,k}| C_{k}$. Then, we are interested in the $P\left(\bigcap_{j}A_{jk}\right) =\prod_{j}P\left(A_{jk}\right)\ge\prod_{j}\prod_{i}\left(\left(1-\delta_{ijk}\right)\right)$. Finally, we have our result:
\begin{eqnarray}
P\left(v\left(\widehat{\Subreg}^{+}_{k} \cap \LO \right)\le N^{+}_{k}\epsilon^{+}_{k} | C_{k}\right)\ge \prod_{j}\prod_{i}\left(1-\delta_{ijk}\right), \forall k.\nonumber
\end{eqnarray}
This completes the proof.
\end{proof}

\begin{lemma}[Positive reclassified error bound]\label{lem::misclassifyRp}
Let $\LO$ denote the $0$-level set defined in~\eqref{eqn::L0}, 
and $\widehat{\Subreg}^{r,+}_{j,k}$ be the set of subregions re-classified from satisfying at the $j$\textsuperscript{th} level of the Part-X tree, at the $k$\textsuperscript{th} iteration of the algorithm. Refer to $\widehat{\Subreg}^{r,+}_{k}$ as the set of subregions re-classified from satisfying at the $k$\textsuperscript{th} iteration of the algorithm.
Assume the event $C_k$ in Definition~\ref{def::misclassE} to be true. Then, the following holds:
\begin{eqnarray} 
P\left(v\left(\widehat{\Subreg}^{r,+}_{j,k} \cap \LO \right)\ge N^{r,+}_{j,k}\epsilon^{+}_{j,k-1} | C_{k}\right)\ge\prod^{N^{r,+}_{j,k}}_{i=1}\left[1-\delta_{ijk}\right], k>1.\label{eqn::volsrpjk}
\end{eqnarray}
where, under event $C^{+}_{k}$, $0 \le \epsilon^{+}_{jk} \le \frac{\epsilon v\left(\widehat{\RegionPlus}_{jk}\right)}{v\left(S\right)}$, and $\delta_{ijk}=\frac{\ln\frac{1}{p_{ijk}}+\ln\frac{1}{\alpha_{jk}}+2\ln n_{jk}+1}{n_{jk}}$. Accounting for all levels yields the following result:
\begin{eqnarray} 
P\left(v\left(\widehat{\Subreg}^{r,+}_{k} \cap \LO \right)\ge N^{r,+}_{k}\epsilon^{+}_{k-1} | C_{k}\right) \ge \prod^{k}_{j=1} \prod^{N^{r,+}_{j,k}}_{i=1}\left[1-\delta_{ijk}\right], k>1.\label{eqn::volsrpk}
\end{eqnarray}
where, under event $C^{+}_{k}$, $0 \le \epsilon^{+}_{k} \le \frac{\epsilon v\left(\widehat{\RegionPlus}_{k}\right)}{v\left(S\right)}$.
\end{lemma}
\begin{proof}
The proof follows the previous and it is therefore omitted here.
\end{proof}

The same results can be proved for the negatively classified volume, leading to lemma~\ref{lem::misclassifym}-\ref{lem::misclassifyRm}, presented below.

\begin{lemma}[Negative classified error bound]\label{lem::misclassifym}
Let $\LO$ denote the $0$-level set defined in~\eqref{eqn::L0}, and $\widehat{\Subreg}^{-}_{jk}$ be the set of subregions classified as violating at the $j$\textsuperscript{th} level of the Part-X tree, at the $k$\textsuperscript{th} iteration of the algorithm. Refer to $\widehat{\Subreg}^{-}_{k}$ as the set of subregions classified as violating at the $k$\textsuperscript{th} iteration of the algorithm. Assume the event $C_k$ in Definition~\ref{def::misclassE} to be true. Then, the following holds:  
\begin{eqnarray} 
P\left(v\left(\hat{\sigma}^-_{jk} \cap \left(S\setminus\LO\right)\right)\le N_{j,k}^-\epsilon^{-}_{j,k}| C_k   \right)\ge \prod^{N^{-}_{j,k}}_{i=1}\left[1-\delta_{ijk}\right]. \label{eqn::volsmjk}
\end{eqnarray}
where, under event $C^{-}_{k}$, $0 \le \epsilon^{-}_{jk} \le \frac{\epsilon v\left(\widehat{\RegionMinus}_{jk}\right)}{v\left(S\right)}$. 
By construction:
\begin{eqnarray} 
P\left(v\left(\hat{\sigma}^-_{k} \cap \left(S\setminus\LO\right)\right)\le N_{k}^-\epsilon^{-}_{k}| C_k   \right)\ge \prod^{k}_{j=1}\prod^{N^{-}_{j,k}}_{i=1}\left[1-\delta_{ijk}\right]. \label{eqn::volsmk}
\end{eqnarray}
where, under event $C^{-}_{k}$, $0 \le \epsilon^{-}_{k} \le \frac{\epsilon v\left(\widehat{\RegionMinus}_{k}\right)}{v\left(S\right)}$.
\end{lemma}
\begin{proof}
The proof is identical to the one of Lemma~\ref{lem::misclassifyRp}, and it is therefore omitted here.
\end{proof}

\begin{lemma}[Negative reclassified error bound]\label{lem::misclassifyRm}
Let $\LO$ denote the $0$-level set defined in~\eqref{eqn::L0}, and $\widehat{\Subreg}^{r,-}_{j,k}$ be the set of subregions re-classified from violating at the $j$\textsuperscript{th} level of the Part-X tree, at the $k$\textsuperscript{th} iteration of the algorithm. Refer to $\widehat{\Subreg}^{r,-}_{k}$ as the set of subregions re-classified from violating at the $k$\textsuperscript{th} iteration of the algorithm.
Assume the event $C_k$ in Definition~\ref{def::misclassE} to be true. Then, the following holds: 
\begin{eqnarray} 
P\left(v\left(\widehat{\Subreg}^{r,-}_{j,k} \cap \left(S\setminus\LO\right) \right)\ge N^{r,-}_{j,k}\epsilon^{-}_{j,k-1} | C_{k-1}\right)\ge \prod^{N^{r,-}_{j,k}}_{i=1}\left[1-\delta_{ijk}\right], k>1.\label{eqn::volsrmjk}
\end{eqnarray}
where, under event $C^{-}_{k}$, $0 \le \epsilon^{-}_{jk} \le \frac{\epsilon v\left(\widehat{\RegionMinus}_{jk}\right)}{v\left(S\right)}$, and $\delta_{ijk}=\frac{\ln\frac{1}{p_{ijk}}+\ln\frac{1}{\alpha_{jk}}+2\ln n_{jk}+1}{n_{jk}}$. By construction:
\begin{eqnarray} 
P\left(v\left(\widehat{\Subreg}^{r,-}_{k} \cap \left(S\setminus\LO\right) \right)\ge N^{r,-}_{k}\epsilon^{-}_{k-1} | C_{k}\right)\ge\prod^{k}_{j=1} \prod^{N^{r,-}_{j,k}}_{i=1}\left[1-\delta_{ijk}\right], k>1.\label{eqn::volsrmk}
\end{eqnarray}
where, under event $C^{-}_{k}$, $0 \le \epsilon^{-}_{k} \le \frac{\epsilon v\left(\widehat{\RegionMinus}_{k}\right)}{v\left(S\right)}$. 
\end{lemma}
\begin{proof}
The proof is identical to the one of Lemma~\ref{lem::misclassifyRp}, and it is therefore omitted here.
\end{proof}

We are now ready to study the cumulated error of the algorithm at each iteration considering both classification and reclassification.

\begin{theorem}\label{thm:seq_misclassifyPlusCondC}
Let $\LO$ and $\epsilon$ be the $0$-level set defined in~\eqref{eqn::L0}, and the maximum tolerated error for Part-X, respectively. Assume the event $C_k$ in Definition~\ref{def::misclassE} to be true. Then, at each iteration $k$ of the algorithm, the following bounds hold for the accumulated probability of mis-classification:
\begin{eqnarray}
P\left(v\left(\widehat{\RegionPlus}_k \cap \LO\right)\le \sum^{k}_{j>1}\left(N^{+}_{j}\epsilon^{+}_{j}-N^{r,+}_{j}\epsilon^{+}_{j-1}\right)|\bigcap_{k}C_{k}\right) \ge \prod_{j=1}^{k} \eta^{+}_j. \label{eqn::thm:seq_misclassifyPlusCondC}
\end{eqnarray}
where $\eta^{+}_j=\left(\prod^{k}_{h=1}\prod^{N^{+}_{j,h}}_{i=1}\left(1-\delta_{ijh}\right)\right)\cdot\left(\prod^{k}_{h=2}\prod^{N^{r,+}_{j,h}}_{i=1}\left(1-\delta_{ijh}\right)\right)$.
\end{theorem}

\begin{proof}
Let us first remind that for $\widehat{\RegionPlus}_k$, and $\widehat{\RegionPlus}_{jk}$, the following holds
\begin{eqnarray}
\widehat{\RegionPlus}_k = \bigcup_{k}\bigcup^{k}_{j=1}\left(\hat{\sigma}^{+}_{jk}\setminus\hat{\sigma}^{r,+}_{jk}\right), \widehat{\RegionPlus}_{jk} = \bigcup_{k}\left(\hat{\sigma}^{+}_{jk}\setminus\hat{\sigma}^{r,+}_{jk}\right).\nonumber
\end{eqnarray}
Consequently, we study
\begin{eqnarray}
P\left(v\left(\bigcup_{k}\bigcup^{k}_{j=1}\left(\hat{\sigma}^{+}_{jk}\setminus\hat{\sigma}^{r,+}_{jk}\right) \cap \LO\right)\le \epsilon|\bigcap_{k}C_{k}\right).\nonumber
\end{eqnarray}
Let us define the event $F_{jk}$, such that if $F_{jk}$ is true, then $v\left(\left(\hat{\sigma}^{+}_{jk}\setminus\hat{\sigma}^{r,+}_{jk}\right) \cap \LO\right)\le \epsilon_j$, where $\epsilon_j=\epsilon\frac{v\left(\widehat{\RegionPlus}_{jk}\right)}{v\left(\widehat{\RegionPlus}_k\right)}$. Here we are particularly interested in the event $F^{k}_{j}$, which looks at the cumulation of error \textit{up to iteration }$k$. Hence, $F^{k}_{j}$ holds if $v\left(\left(\bigcup^{k}_{h=1}\hat{\sigma}^{+}_{jh}\right) \cap \LO\right)-v\left(\left(\bigcup^{k}_{h=1}\hat{\sigma}^{r,+}_{jh} \right) \cap \LO \right)\le \epsilon_j$. From Lemma~\ref{lem::misclassifyRp}-\ref{lem::misclassifyP}, we can derive a upper bound for $v\left(\left(\bigcup^{k}_{h=1}\hat{\sigma}^{+}_{jh}\right) \cap \LO\right)$ and a lower bound for $v\left(\left(\bigcup^{k}_{h=1}\hat{\sigma}^{r,+}_{jh} \right) \cap \LO \right)$.

For the reclassified volume, given an iteration $h$, we know the following from Lemma~\ref{lem::misclassifyP}-\ref{lem::misclassifyRp}:
\begin{eqnarray} 
P\left(v\left( \hat{\sigma}^{+}_{jh} \cap \LO\right)-v\left(\widehat{\Subreg}^{r,+}_{jh} \cap \LO \right)\le N^{+}_{jh}\epsilon^{+}_{j,h}-N^{r,+}_{jh}\epsilon^{+}_{j,h-1} | C_{h}\right)\ge \prod^{N^{+}_{j,h}}_{i=1}\left(1-\delta_{ijh}\right)\prod^{N^{r,+}_{j,h}}_{i=1}\left(1-\delta_{ijh}\right), h>1.\nonumber
\end{eqnarray}
Then if we want to consider the iterations $h=1,\ldots,k$, we obtain the following, assuming the event $\bigcap^{k}_{h=1}C_{h}$ holds:
\begin{eqnarray} 
P\left(v\left(\left(\bigcup^{k}_{h=1}\hat{\sigma}^{+}_{jh}\right) \cap \LO\right)-v\left(\left(\bigcup^{k}_{h=1}\widehat{\Subreg}^{r,+}_{jh}\right) \cap \LO \right)\le N^{+}_{jh}\epsilon^{+}_{j,h}-N^{r,+}_{jh}\epsilon^{+}_{j,h-1} | \bigcap^{k}_{h=1}C_{h}\right) \ge \nonumber\\ \left(\prod^{k}_{h=1}\prod^{N^{+}_{j,h}}_{i=1}\left(1-\delta_{ijh}\right)\right)\cdot\left(\prod^{k}_{h=2}\prod^{N^{r,+}_{j,h}}_{i=1}\left(1-\delta_{ijh}\right)\right) = \eta^{+}_j. \nonumber
\end{eqnarray}
Then, considering the condition to hold at \textit{all }levels, we have:
\begin{eqnarray}
P\left(v\left(\widehat{\RegionPlus}_k \cap \LO\right)\le \sum^{k}_{j>1}\left(N^{+}_{j}\epsilon^{+}_{j}-N^{r,+}_{j}\epsilon^{+}_{j-1}\right)|\bigcap_{k}C_{k}\right) \ge \prod_{j=1}^{k} \eta^{+}_j.\nonumber
\end{eqnarray}
\end{proof}

\begin{theorem}\label{thm:seq_misclassifyMinusCondC}
Let $\LO$ and $\epsilon$ be the $0$-level set defined in~\eqref{eqn::L0}, and the maximum tolerated error for Part-X, respectively. Assume the event $C_k$ in Definition~\ref{def::misclassE} to be true. Then, at each iteration $k$ of the algorithm, the following bounds hold for the accumulated probability of mis-classification:
\begin{eqnarray}
P\left(v\left(\widehat{\RegionMinus}_k\setminus\left(\widehat{\RegionMinus}_k \cap \LO\right)\right)\le \sum^{k}_{j>1}\left(N^{-}_{j}\epsilon^{-}_{j}-N^{r,-}_{j}\epsilon^{-}_{j-1}\right)|\bigcap_{k}C_{k}\right) \ge \prod_{j=1}^{k}\eta^{-}_j.\label{eqn::thm:seq_misclassifyMinusCondC} 
\end{eqnarray}
where $\eta^{-}_j=\left(\prod^{k}_{h=1}\prod^{N^{-}_{j,h}}_{i=1}\left(1-\delta_{ijh}\right)\right)\cdot\left(\prod^{k}_{h=2}\prod^{N^{r,-}_{j,h}}_{i=1}\left(1-\delta_{ijh}\right)\right)$.
\end{theorem}

\begin{proof}
The proof is identical to the one of Theorem~\ref{thm:seq_misclassifyPlusCondC}, and it is therefore omitted here.
\end{proof}

Note that $\epsilon = \sum^{k}_{j>1}\left(N^{+}_{j}\epsilon^{+}_{j}-N^{r,+}_{j}\epsilon^{+}_{j-1}\right)+\sum^{k}_{j>1}\left(N^{-}_{j}\epsilon^{-}_{j}-N^{r,-}_{j}\epsilon^{-}_{j-1}\right)$ is the maximum overall error allowed over the input space, the study of the error will be shown in the analysis of the density associated to the event $C_{h}$ in Lemma~\ref{thm::misclEDistrib}. Before the error event analysis, let us provide two results that follow from the previous theorems.

\begin{corollary}[Per iteration positive classification error]\label{cor::posVolk}
Let $\LO$ and $\epsilon$ be the $0$-level set defined in~\eqref{eqn::L0}, and the maximum tolerated error for Part-X, respectively. Assume the event $C_k$ in Definition~\ref{def::misclassE} to be true. Then, at each iteration $k$ of the algorithm, the following bounds hold for the probability of mis-classification of positive volume at iteration $k$:
\begin{eqnarray} \label{eqn::cor:seq_misclassifyPlusCondCK}
P\left(v\left(\widehat{\sigma}^{+}_k \cap \LO\right) - v\left(\widehat{\sigma}^{r+}_k \cap \LO\right)\le N^{+}_{k}\epsilon^{+}_{k}-N^{r,+}_{k}\epsilon^{+}_{k-1} |C_{k}\right) \ge \prod_{j=1}^{k}\left( \prod^{N^{+}_{j,k}}_{i=1}\left(1-\delta_{ijk}\right)\prod^{N^{r,+}_{j,k}}_{i=1}\left(1-\delta_{ijk}\right)\right)=\gamma^{+}_{k}, k>1.
\end{eqnarray}
\end{corollary}

Similarly, for the negatively classified volume, we have:
\begin{corollary}[Per iteration negatively classification error]\label{cor::negVolk}
Let $\LO$ and $\epsilon$ be the $0$-level set defined in~\eqref{eqn::L0}, and the maximum tolerated error for Part-X, respectively. Assume the event $C_k$ in Definition~\ref{def::misclassE} to be true. Then, at each iteration $k$ of the algorithm, the following bounds hold for the probability of mis-classification of negative volume at iteration $k$:
\begin{eqnarray} \label{eqn::cor:seq_misclassifyMinusCondCK}
P\left(v\left(\widehat{\sigma}^{-}_k \cap \left(S\setminus\LO\right)\right) - v\left(\widehat{\sigma}^{r-}_k \cap \left(S\setminus\LO\right)\right)\le N^{-}_{k}\epsilon^{-}_{k}-N^{r,-}_{k}\epsilon^{-}_{k-1} |C_{k}\right)\ge\nonumber\\
\ge\prod_{j=1}^{k}\left(\prod^{N^{-}_{j,k}}_{i=1}\left(1-\delta_{ijk}\right)\prod^{N^{r,-}_{j,k}}_{i=1}\left(1-\delta_{ijk}\right)\right)=\gamma^{-}_{k}, k>1.\nonumber
\end{eqnarray}
\end{corollary}

\begin{lemma}\label{thm::misclEDistrib} 
Let $C_{k}=C^{+}_{k} \cap C^{-}_{k}$ be the event that, at any iteration $k > 0$ of Part-X, the volume of the incorrectly classified subregions is smaller than $\epsilon$. Let $n$ be the number of function evaluations in each, remaining, subregion up to iteration $k$. Then, the following bound holds for the probability of the event $\cap^{k}_{j=1}C_{j}$
\begin{eqnarray}
P\left(\bigcap^{k}_{j=1} \left\lbrace C^{+}_{j} \cap C^{-}_{j}\right\rbrace \right)\ge \prod^{k}_{h>1}\gamma^{+}_{h}\cdot\prod^{k}_{h>1}\gamma^{-}_{h}.\label{lem::misclEDistrib}
\end{eqnarray} 
\end{lemma}
\begin{proof}
By definition we have:
\begin{eqnarray}
P\left(C_{k}\right) = & P\left(\left\lbrace v\left(\widehat{\RegionPlus}_k\cap \LO\right)\le \frac{\epsilon v\left(\widehat{\RegionPlus}_k\right)}{v\left(S\right)}\right\rbrace\bigcap\left\lbrace v\left(\widehat{\RegionMinus}_k \cap \LO\right) \ge v\left(\widehat{\RegionMinus}_k\right)-\frac{\epsilon v(\widehat{\RegionMinus}_k)}{v\left(S\right)}\right\rbrace\right)\nonumber\\
= & P\left(C^{+}_k\cap C^{-}_{k}\right).\label{eqn::eventCkiter}
\end{eqnarray}
Now let's study the probability that the event in~\eqref{eqn::eventCkiter} holds \textit{at each iteration}. Now, we work toward the estimation of the probability that the estimated $0$-level set captures the true level set, that is:
\begin{eqnarray}
P\left(\bigcap^{k}_{h=1}\left\lbrace v\left(\widehat{\RegionPlus}_h\cap \LO\right) + v\left(\widehat{\RegionMinus}_h \cap \left(S\setminus\LO\right)\right) \right\rbrace \le \epsilon \right)
\ge P\left(\bigcap^{k}_{h=1}\left\lbrace C^{+}_h\cap C^{-}_{h}\right\rbrace\right).\nonumber
\end{eqnarray}
We show that:
\begin{eqnarray}
P\left(\bigcap^{k}_{h=1}\left\lbrace C^{+}_h\cap C^{-}_{h}\right\rbrace\right)\ge \prod^{k}_{h>1}\gamma^{+}_{h}\prod^{k}_{h>1}\gamma^{-}_{h}.\label{eqn::toshow}
\end{eqnarray}
The result in~\eqref{eqn::toshow} holds for the first iteration since the classified sets are empty. Let us now assume that~\eqref{eqn::toshow} holds for $k=j$, i.e.:
\begin{eqnarray}
P\left(\bigcap^{j}_{h=1}\left\lbrace C^{+}_h\cap C^{-}_{h}\right\rbrace\right) = \prod^{j}_{h>1}\gamma^{+}_{h}\prod^{j}_{h>1}\gamma^{-}_{h}.\label{eqn::holdsforj}
\end{eqnarray}
Then, for $k=j+1$, we have:
\begin{eqnarray}
P\left(\bigcap^{j+1}_{h=1}\left\lbrace C^{+}_h\cap C^{-}_{h}\right\rbrace\right) = P\left(\left\lbrace C^{+}_{j+1}\cap C^{-}_{j+1}\right\rbrace|\bigcap^{j}_{h=1}\left\lbrace C^{+}_{h}\cap C^{-}_{h}\right\rbrace\right)\cdot P\left(\bigcap^{j}_{h=1}\left\lbrace C^{+}_h\cap C^{-}_{h}\right\rbrace\right)\ge\nonumber\\
\ge P\left(C^{+}_{j+1}|\bigcap^{j}_{h=1}\left\lbrace C^{+}_{h}\cap C^{-}_{h}\right\rbrace\right)\cdot P\left(C^{+}_{j+1}|\bigcap^{j}_{h=1}\left\lbrace C^{+}_{h}\cap C^{-}_{h}\right\rbrace\right)\cdot P\left(\bigcap^{j}_{h=1}\left\lbrace C^{+}_h\cap C^{-}_{h}\right\rbrace\right).\label{eqn::partResulLem8}
\end{eqnarray}
Consider the Corollary~\ref{cor::posVolk}-\ref{cor::negVolk} that bound the error for the classified volume at a specific iteration. In particular, since $\widehat{\RegionPlus}_{j+1}=\widehat{\RegionPlus}_{j}\cup\hat{\Subreg}^{+}_{j}\setminus \hat{\Subreg}^{r+}_{j},\widehat{\RegionMinus}_{j+1}=\widehat{\RegionMinus}_{j}\cup\hat{\Subreg}^{-}_{j}\setminus \hat{\Subreg}^{r-}_{j}$, we have:
\begin{eqnarray}
\eqref{eqn::partResulLem8}\ge P\left(v\left(\hat{\Subreg}^{+}_{j}\cap\LO\right)-v\left(\hat{\Subreg}^{r+}_{j}\cap \LO\right)\le\Delta^{+,\epsilon}_j|\bigcap^{j}_{h=1}\left\lbrace C^{+}_{h}\cap C^{-}_{h}\right\rbrace\right)\cdot\nonumber\\
\cdot P\left(v\left(\hat{\Subreg}^{-}_{j}\cap\left(S\setminus\LO\right)\right)-v\left(\hat{\Subreg}^{r-}_{j} \cap\left(S\setminus\LO\right)\right)\le\Delta^{-,\epsilon}_j|\bigcap^{j}_{h=1}\left\lbrace C^{+}_{h}\cap C^{-}_{h}\right\rbrace\right)\cdot P\left(\bigcap^{j}_{h=1}\left\lbrace C^{+}_h\cap C^{-}_{h}\right\rbrace\right).\label{eqn::partResul2Lem8}
\end{eqnarray}
where $\Delta^{+,\epsilon}_j = N^{+}_{j}\epsilon^{+}_{j}-N^{r+}_{j}\epsilon^{+}_{j-1}$, and $\Delta^{-,\epsilon}_j = N^{-}_{j}\epsilon^{-}_{j}-N^{r-}_{j}\epsilon^{-}_{j-1}$. Then, considering Corollary~\ref{cor::posVolk}-\ref{cor::negVolk}, and the assumption~\eqref{eqn::holdsforj}, we have:
\begin{eqnarray}
\eqref{eqn::partResul2Lem8}\ge \gamma^{+}_{j+1}\cdot \gamma^{-}_{j+1}\cdot \prod^{j}_{h>1}\gamma^{+}_{h}\prod^{j}_{h>1}\gamma^{-}_{h}=\prod^{j+1}_{h>1}\gamma^{+}_{h}\prod^{j+1}_{h>1}\gamma^{-}_{h}.\nonumber
\end{eqnarray}
Thus proving the claim.
\end{proof}

\begin{theorem}\label{thm:seq_misclassifyPlus}
Let $\LO$ and $\epsilon$ be the $0$-level set defined in~\eqref{eqn::L0}, and the maximum tolerated error for Part-X, respectively. Then, at each iteration $k$ of the algorithm, the following bounds hold for the accumulated probability of mis-classification:
\begin{eqnarray} \label{eqn::thm:seq_misclassifyPlus}
P\left(v\left(\widehat{\RegionPlus}_k \cap \LO\right)\le \epsilon^{+}\right) \ge \prod^{k}_{h>1}\gamma^{-}_{h}\cdot\prod^{k}_{h>1}\left(\gamma^{+}_{h}\right)^{2}.\label{thm::overallErrP}
\end{eqnarray}
\end{theorem}

\begin{proof}
The result follows from Lemma~\ref{thm::misclEDistrib}.
\end{proof}
Similarly, for the negatively classified volume:
\begin{theorem}\label{thm:seq_misclassifyMinus}
Let $\LO$ and $\epsilon$ be the $0$-level set defined in~\eqref{eqn::L0}, and the maximum tolerated error for Part-X, respectively. Assume the event $C_k$ in Definition~\ref{def::misclassE} to be true. Then, at each iteration $k$ of the algorithm, the following bounds hold for the accumulated probability of mis-classification:
\begin{eqnarray} \label{eqn::thm:seq_misclassifyMinus}
P\left(v\left(\widehat{\RegionMinus}_k\right)-v\left(\widehat{\RegionMinus}_k \cap \LO\right)\le \epsilon^{-}\right) \ge \prod^{k}_{h>1}\gamma^{+}_{h}\cdot\prod^{k}_{h>1}\left(\gamma^{-}_{h}\right)^{2}.\label{thm::overallErrM}
\end{eqnarray}
\end{theorem}

\begin{proof}
The result follows from Lemma~\ref{thm::misclEDistrib}.
\end{proof}

\section{Numerical Results}\label{sec::results}

We have implemented Part-X as a Python library that can be used as a stand-alone level-set approximation function, or as a module within our SBTG tool.
In Section \ref{sec::empirical}, we first present results on using Part-X on benchmark functions from the optimization literature.   
Two of these functions are 2 dimensional, which is convenient for plotting the zero-level sets.
Then, in Section \ref{sec::falsification}, we apply our SBTG tool on the F-16 benchmark problem~\cite{heidlauf2018verification}.

Part-X produces among the outputs the \textit{falsification volume}, which we define below and further characterize in the analysis.

\begin{definition}[Falsification Volume]
Let $\left\lbrace\Subreg^{\gamma}_{ijk}\right\rbrace$ be the subregions produced by the Part-X algorithm up to the $k$\textsuperscript{th} iteration each with associated surrogate model $\gpPred\left(\mbf{x}\right), \mbf{x}\in\Subreg^{\gamma}_{ijk}$, then, the falsification volume calculated with $\Significance$-quantile results:
\begin{eqnarray}
V^{\Significance}|\left(\mathbf{x},\mathbf{y}\right) = \int_{\bigcup\Subreg^{\gamma}_{ijk}}I_{\underbar{q}_{\Significance}<0}.
\end{eqnarray}
\end{definition}

\subsection{Non Linear Non Convex Optimization Examples}\label{sec::empirical}
In this section, we test Part-X for level set estimation with several non-linear non-convex test functions: the Rosenbrock, Goldstein, and Himmelblau's functions, defined as:
\begin{itemize}
    \item Rosenbrock's function $\left(-1 \le x_i \le 1, i = 1,\ldots d\right)$
    \begin{eqnarray}
    f\left(\mathbf{x}\right) =\sum_{i=1}^{d-1}\left\lbrace 100\left(x_{i+1} - x_i^{2}\right)^2 + \left(1-x_i\right)^2 \right\rbrace - 20\nonumber
    \end{eqnarray}
    this function has symmetric level set in region $\Designspace = \left[-1,1\right]\times\left[-1,1\right]$.
    \item Goldstein-Price function $\left(-1 \le x,y \le 1 \right)$:
    \begin{dmath*}
    f\left(x,y\right) = \left\lbrace 1 + \left(x+y+1\right)^2 \left(19-14x+3x^2-14y+6xy+3y^2\right) \right\rbrace
    \left\lbrace 30 +\left(2x - 3y \right)^2\left(18-32x+12x^2+48y-36xy+27y^2\right) \right\rbrace - 50.
    \end{dmath*}
    The Goldstein function has smaller volume $\LO$ compared to the Rosenbrock and the domain is $\Designspace = \left[-1,1\right]\times\left[-1,1\right]$,
    \item Himmelblau's function $\left(-5 \le x,y \le 5 \right)$
    \begin{eqnarray}
    f\left(x,y\right)= \left(x^2+y-11\right)^2 +\left(x+y^2 -7\right)^2 -40.\nonumber
    \end{eqnarray}
     The level set $\LO$ for the Himmelblau's is disconnected and the function support is $\Designspace = \left[-5,5\right]\times\left[-5,5\right]$. 
\end{itemize} 

Part-X was ran with initialization budget $n_0 = 10$, unclassified subregions, per-subregion, budget $n_{\mbox{\tiny{BO}}} = 10$, classified subregions budget $n_c=100$, maximum budget $T=5000$, number of Monte Carlo iterations $R=10$, number of evaluations per iterations $M=100$, number of cuts $B = 2$,
 classification percentile $\Significance = 0.05$. Also, we used $\minVol = 0.001$ to define unbranchable dimensions. 

Figs.~\ref{fig::rosenS1}-\ref{fig::rosenS3} show the partitioning obtained from Part-X from a randomly chosen macroreplication. The black solid lines identify the contour of the true level set $\mathcal{L}_0$. The red subregions are classified by Part-X as violating, while the green subregions are classified as satisfying the requirements. We can see that the algorithm has quite good performance in correctly identifying the relevant region. We observe that, as the grid size increases, from Fig.~\ref{fig::rosenS1} to~\ref{fig::rosenS3}, no clear pattern is observed from the results. This is a positive outcome, since it shows the Robustness of Part-X with respect to the size of the grid adopted.
Observing the pattern formed by the sampled locations, we can see how samples concentrate around the border of the level set, as expected, while areas away from the border receive much less effort. Similar observations can be gathered from Figs.~\ref{fig::goldS1}-\ref{fig::goldS3}.

\begin{figure}[htbp]
\begin{center}
    \subfigure[Grid Size $R\times M=10\times 100$. \label{fig::rosenS1}]{
    \includegraphics[width=0.28\textwidth]{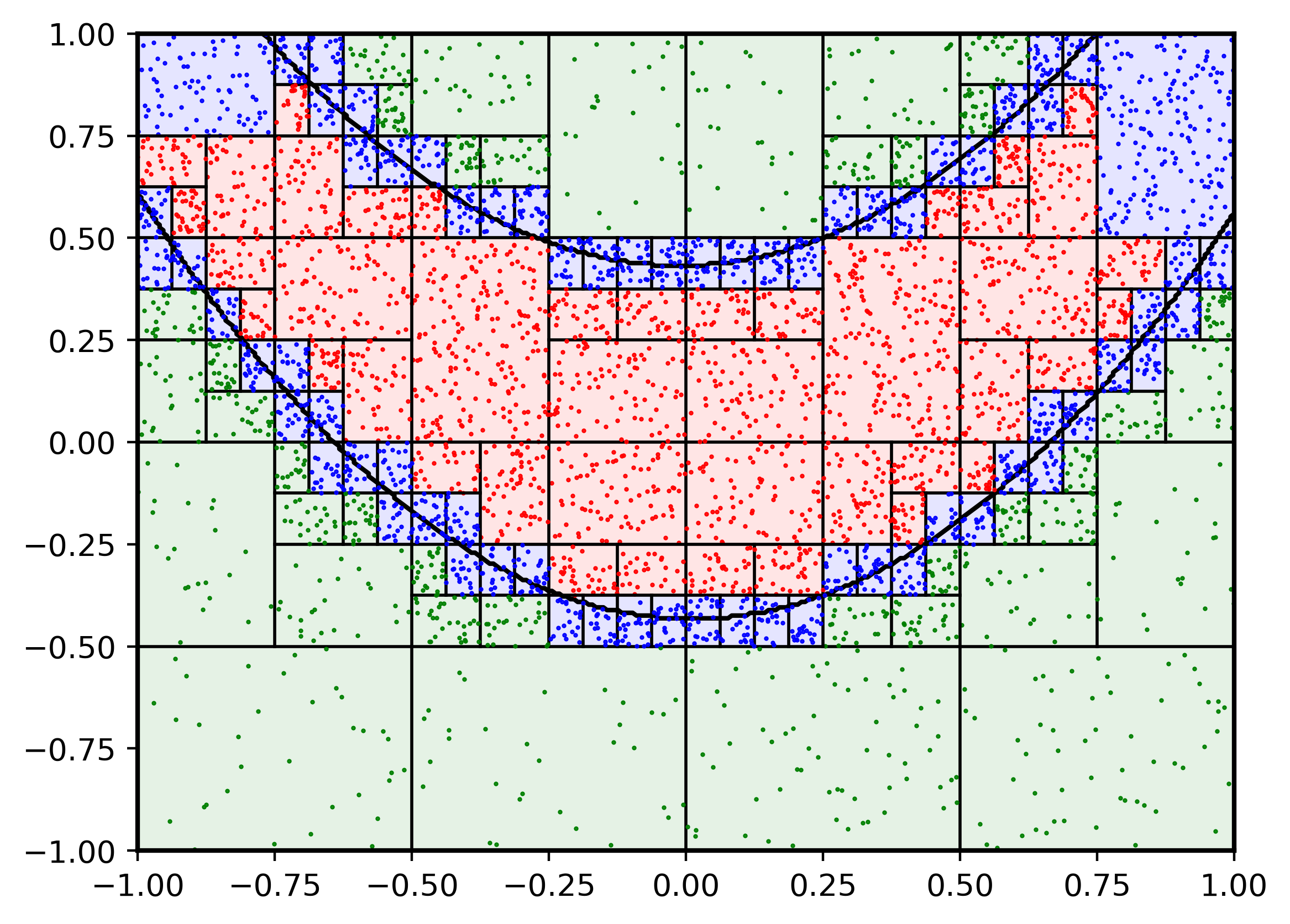}}\quad
    \subfigure[Grid Size $R\times M=10\times 500$. \label{fig::rosenS2}]{
    \includegraphics[width=0.28\textwidth]{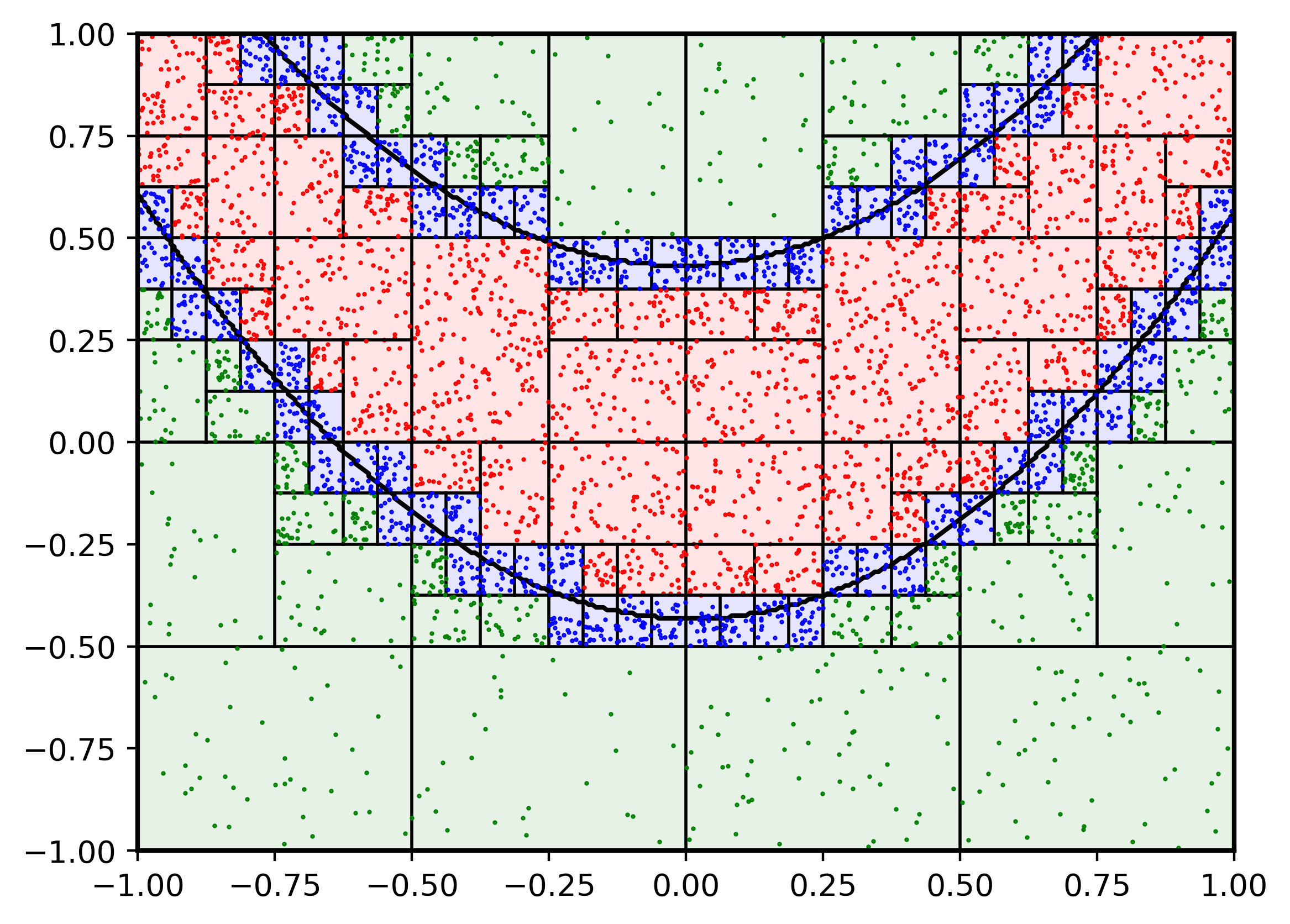}}\quad
    \subfigure[Grid Size $R\times M=20\times 500$. \label{fig::rosenS3}]{
    \includegraphics[width=0.28\textwidth]{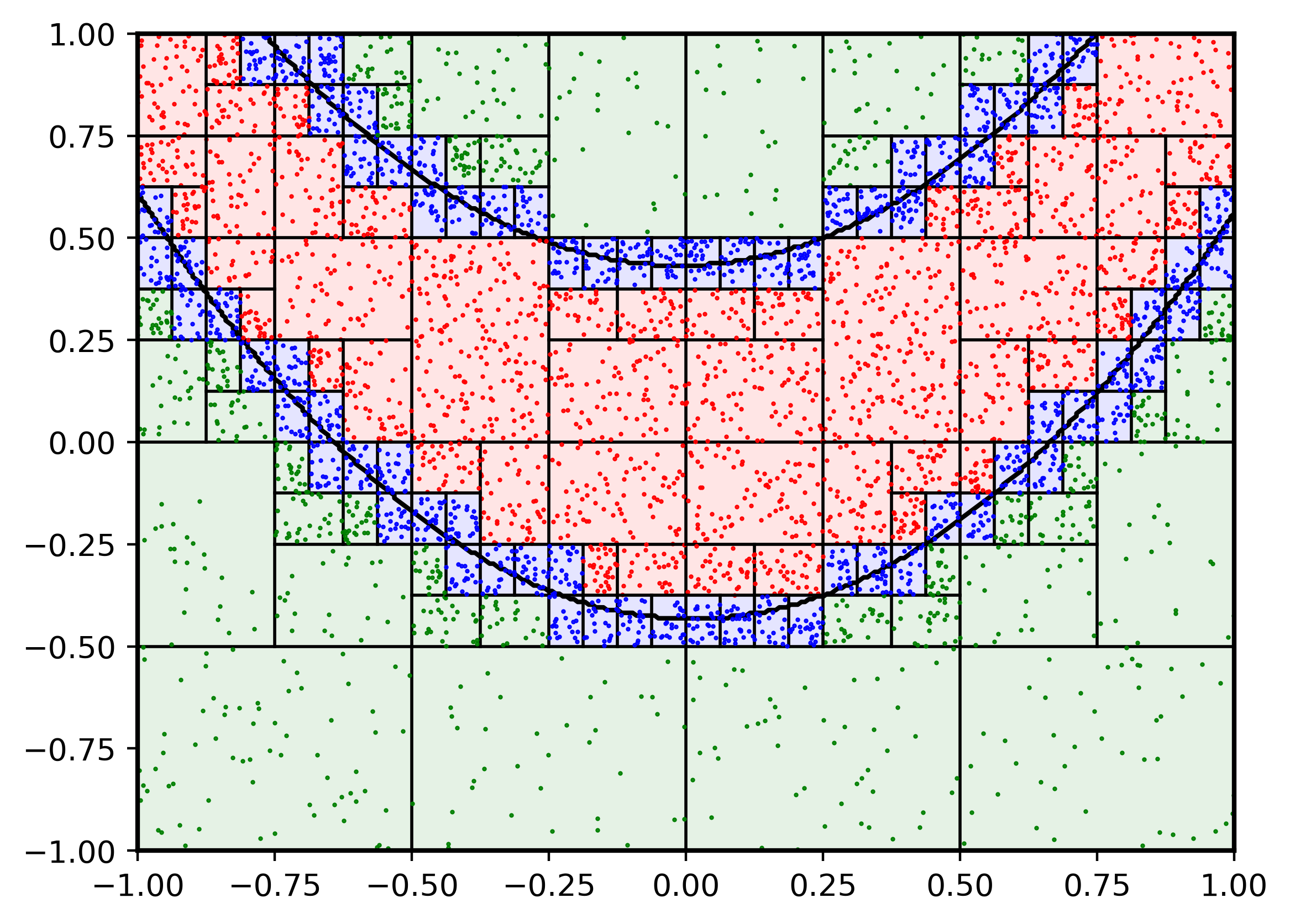}}
    \caption{Sampling pattern for the Rosenbrock 0-level set identification.}\label{fig::samplelocationRosen}
\end{center}
\end{figure}
\begin{figure}[htbp]
\centering
\subfigure[Grid Size $R\times M=10\times 100$. \label{fig::goldS1}]{
\includegraphics[width=0.28\textwidth]{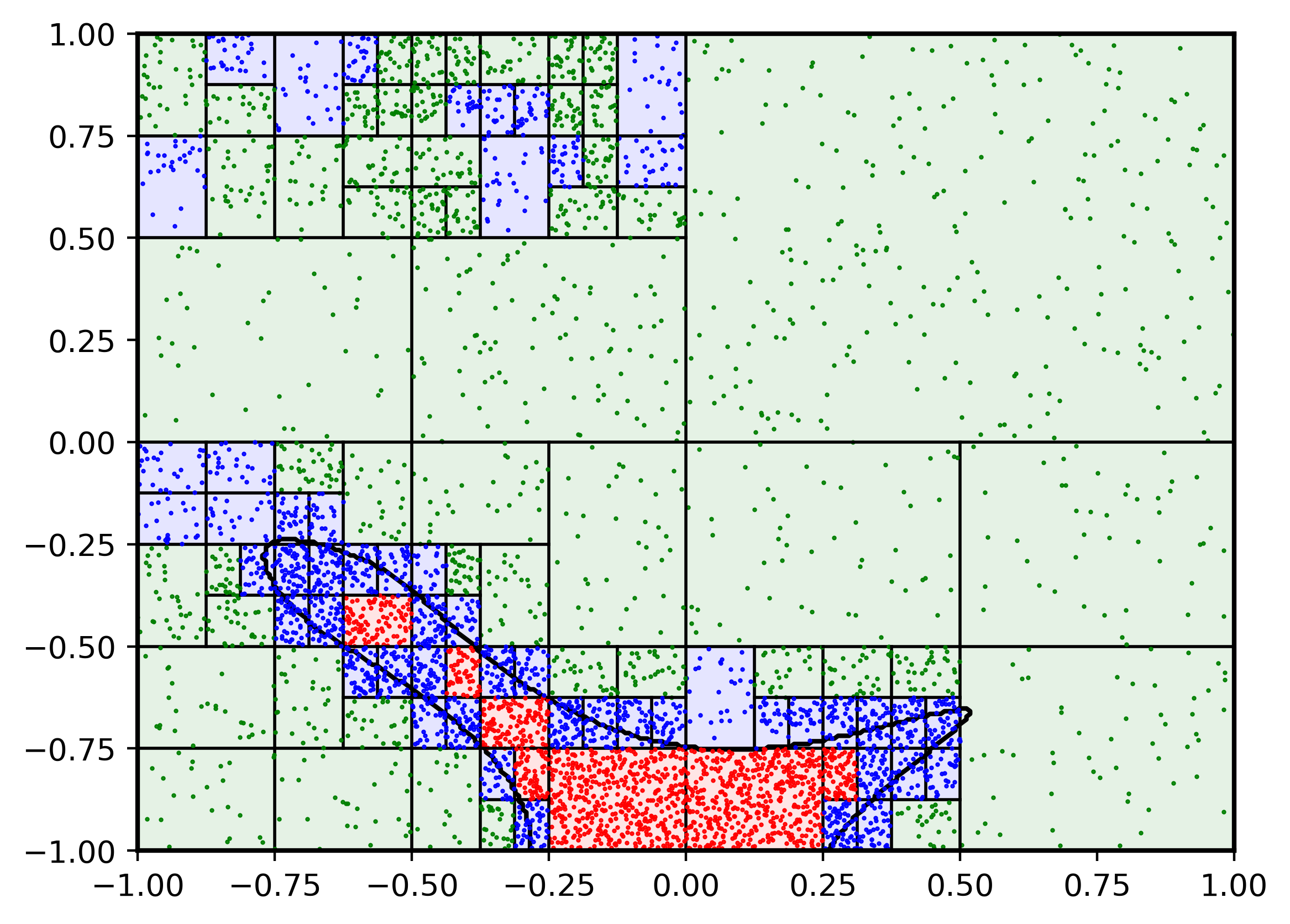}}\quad
\subfigure[Grid Size $R\times M=10\times 500$. \label{fig::goldS2}]{
\includegraphics[width=0.28\textwidth]{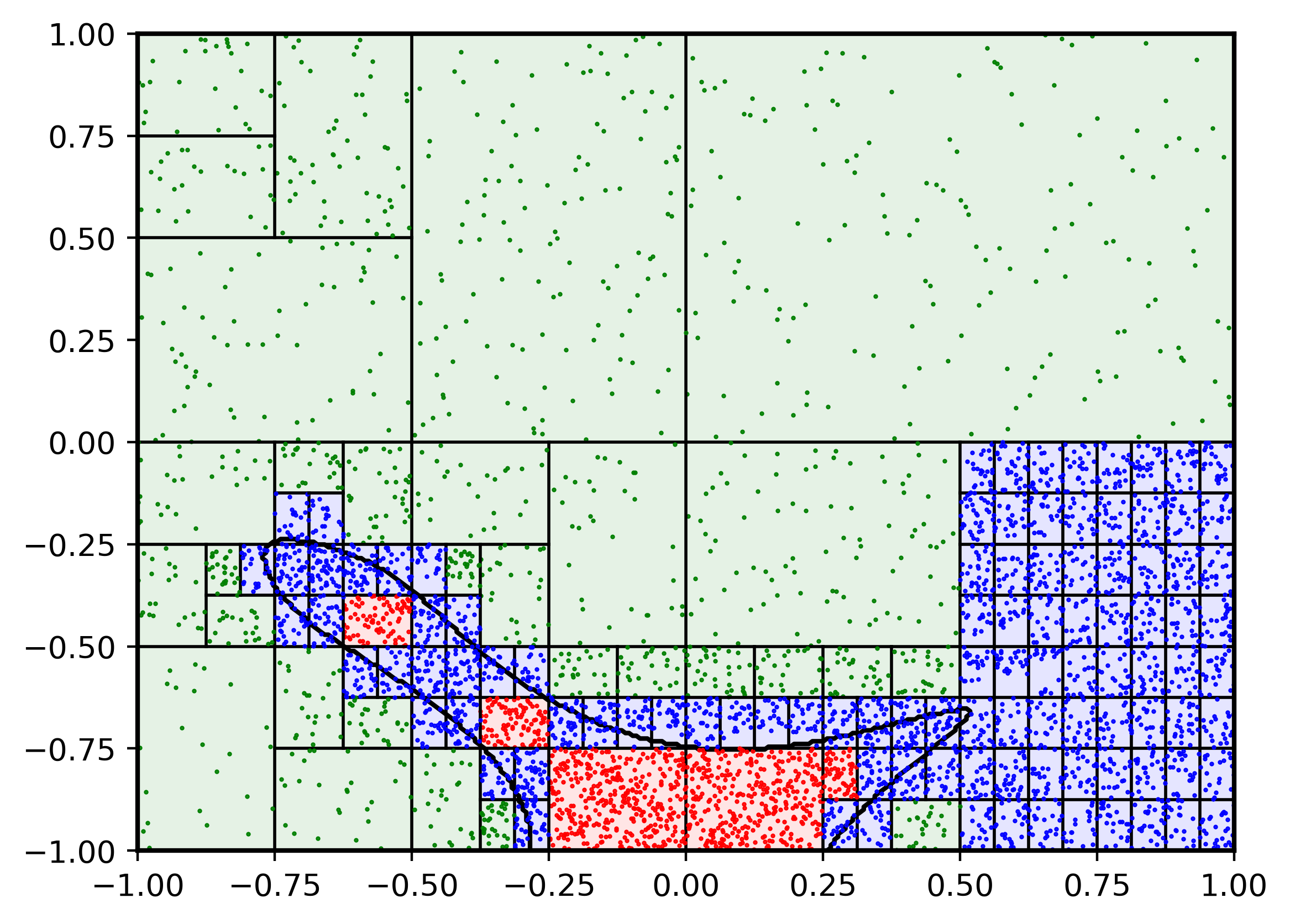}}\quad
\subfigure[Grid Size $R\times M=20\times 500$. \label{fig::goldS3}]{
\includegraphics[width=0.28\textwidth]{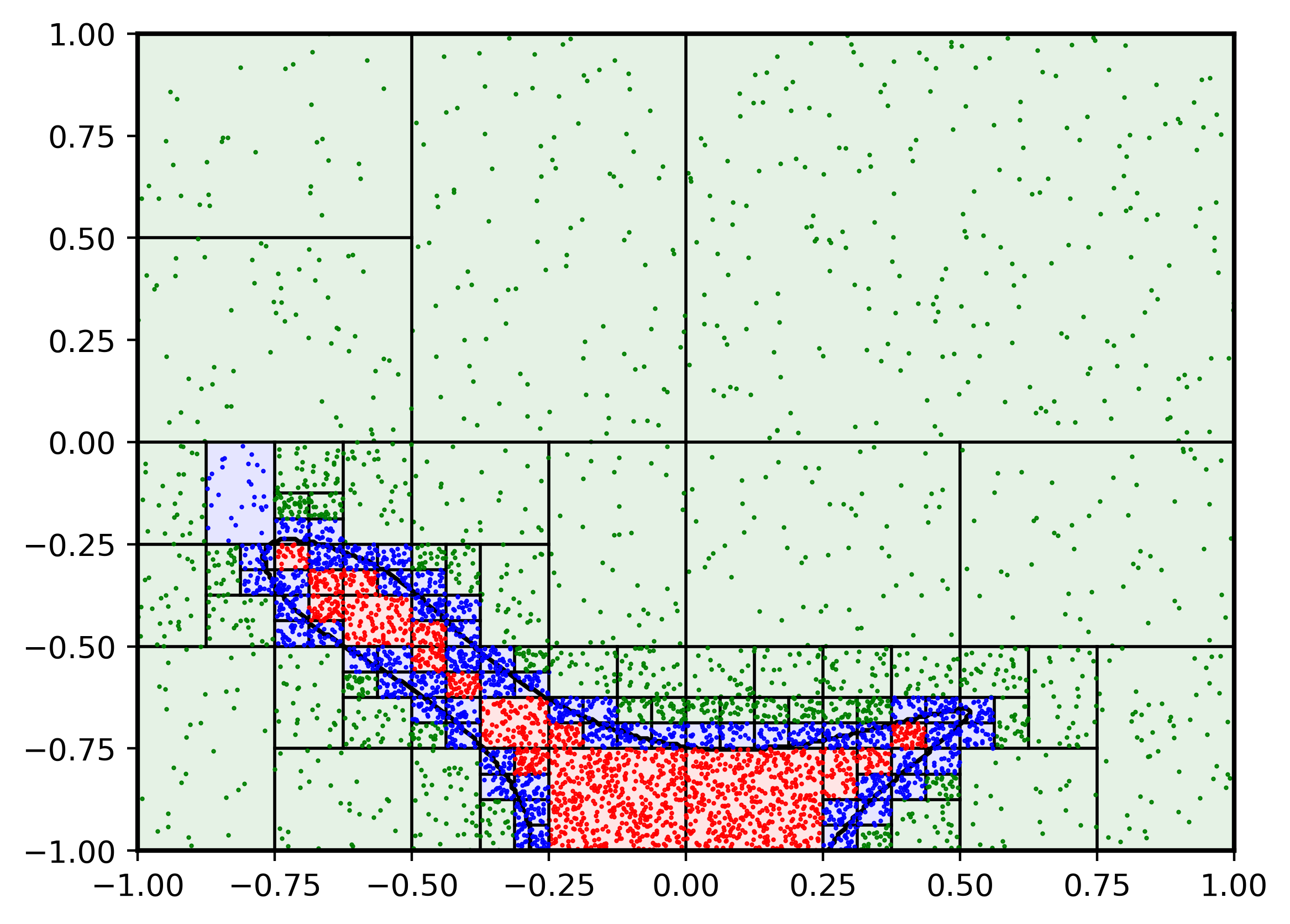}}
\caption{Sampling pattern for the Goldstein-Price 0-level set identification.}\label{fig::samplelocationGold}
\end{figure}
\begin{figure}[htbp]
\centering
\subfigure[Grid Size $R\times M=10\times 100$. \label{fig::himmS1}]{
\includegraphics[width=0.28\textwidth]{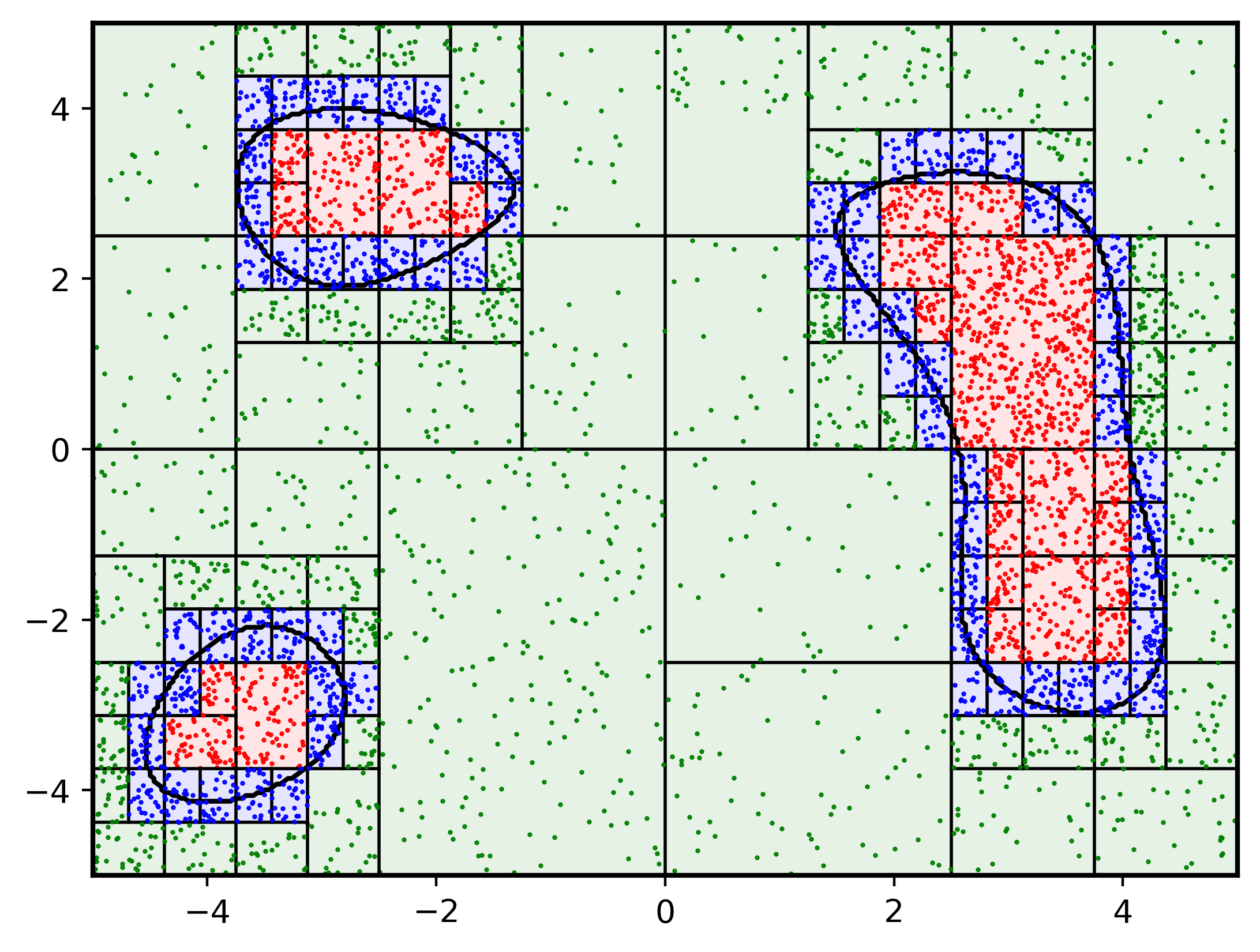}}
\subfigure[Grid Size $R\times M=10\times 500$. \label{fig::himmS2}]{
\includegraphics[width=0.28\textwidth]{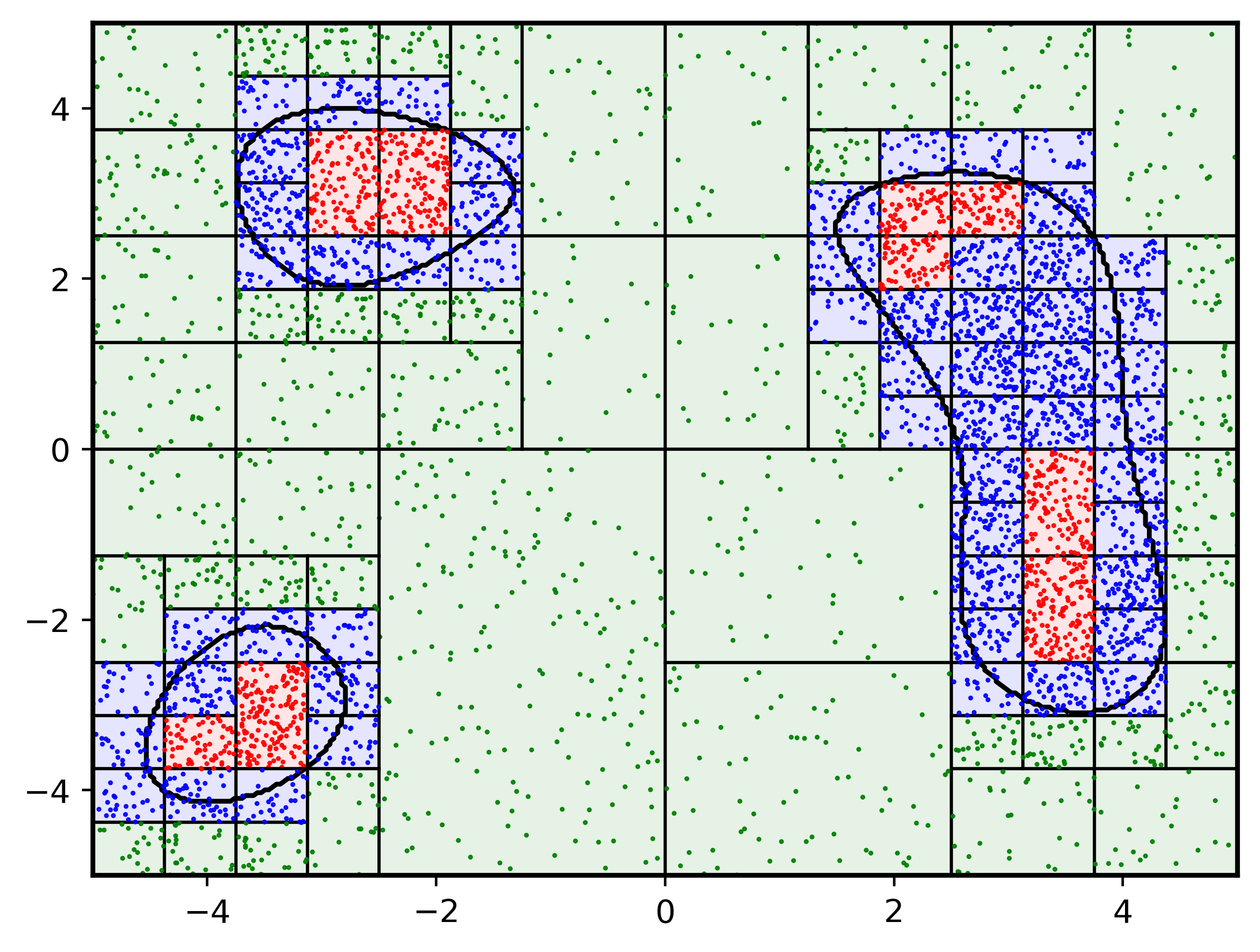}}
\subfigure[Grid Size $R\times M=20\times 500$. \label{fig::himmS3}]{
\includegraphics[width=0.28\textwidth]{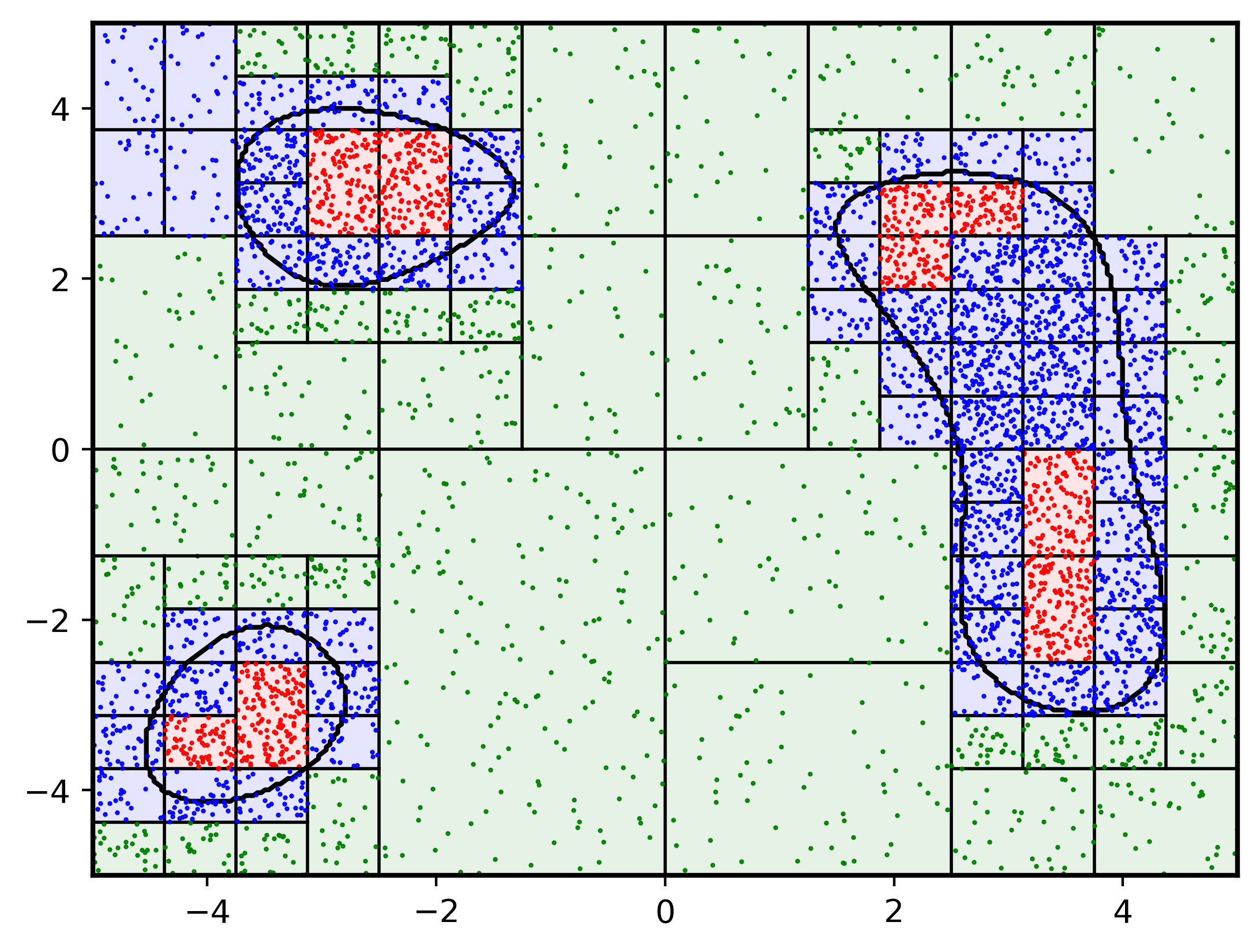}}
\caption{Sampling pattern for the Himmellblau 0-level set identification.}\label{fig::samplelocationHimm}
\end{figure}

Tables~\ref{tab:resVaolR10M100}-\ref{tab:resVaolR20M500} show the results for the falsification volume $\FalsVol$ obtained using: (1) the sum of the violating and remaining volumes; (2) the estimate obtained with the Gaussian process prediction using the $\falsq$-quantile. Results are reported for each of the three test functions. The standard error for the metric was obtained as a result of the $50$ macro replications of the algorithm. 
The Monte Carlo estimation (Monte Carlo row in Tables~\ref{tab:resVaolR10M100}-\ref{tab:resVaolR20M500}) is the point estimate for the falsification volume obtained with $R\times M$ function evaluations. We can observe that using the Gaussian process as a means to estimate of the violating volume produces accurate estimates and results robust with respect to the selected $\Significance$-quantile. This is due to the fact that we achieve particularly high density of sampling in violating regions so that the model variance is also very low within the level set. 

\begin{table}[H]
  \centering
  \caption{Falsification volume obtained with $R=10, M=100$. We estimate the falsifying volume as the volume of the remaining and violating hyperboxes (P-X). P-X ($\falsq$) algorithm variants use the Gaussian processes $\falsq$-quantile to estimate the falsifying volume. Monte Carlo uses the same number of evaluations ($5000\times 30$) to perform a one-shot estimation of the falsifying volume, so no standard error is provided.}
    \begin{tabular}{lcccccc}
    \toprule
    \multicolumn{1}{c}{\multirow{2}[0]{*}{Algorithm}} & \multicolumn{2}{c}{Rosenbrock} & \multicolumn{2}{c}{Goldstein} & \multicolumn{2}{c}{Himmellblau} \\
          & \multicolumn{1}{l}{Mean} & \multicolumn{1}{l}{std err} & \multicolumn{1}{l}{Mean} & \multicolumn{1}{l}{std err} & \multicolumn{1}{l}{Mean} & \multicolumn{1}{l}{std err} \\
    \hline
    P-X & 1.901 & \multicolumn{1}{r}{0.05E-02} & 1.176 & \multicolumn{1}{r}{6.75E-02} & 31.953 & \multicolumn{1}{r}{25.122} \\
    P-X ($\falsq=0.5$) & 1.628 & \multicolumn{1}{r}{1.43E-07} & 0.304 & \multicolumn{1}{r}{1.27E-07} & 17.671 & \multicolumn{1}{r}{9.37E-05} \\
    P-X ($\falsq=0.95$) & 1.628 & \multicolumn{1}{r}{1.43E-07} & 0.304 & \multicolumn{1}{r}{1.27E-07} & 17.671 & \multicolumn{1}{r}{9.37E-05} \\
    P-X ($\falsq=0.99$) & 1.628 & \multicolumn{1}{r}{1.43E-07} & 0.304 & \multicolumn{1}{r}{1.27E-07} & 17.671 & \multicolumn{1}{r}{9.37E-05} \\
    \hline
    Monte Carlo & \multicolumn{2}{c}{1.626} & \multicolumn{2}{c}{0.302} & \multicolumn{2}{c}{17.030} \\
    \bottomrule
    \end{tabular}%
  \label{tab:resVaolR10M100}%
\end{table}%

Table~\ref{tab:resVaolR10M500} shows the results obtained with the larger grid obtained with $R=10, M=500$.  

\begin{table}[H]
  \centering
  \caption{Falsification volume obtained with $R=10, M=500$. We estimate the falsifying volume as the volume of the remaining and violating hyperboxes (P-X). P-X ($\falsq$) algorithm variants use the Gaussian processes $\falsq$-quantiles to estimate the falsifying volume. Monte Carlo uses the same number of evaluations ($5000\times 30$) to perform a one-shot estimation of the falsifying volume, so no standard error is provided.}
    \begin{tabular}{lcccccc}
    \toprule
    \multicolumn{1}{c}{\multirow{2}[0]{*}{Algorithm}} & \multicolumn{2}{c}{Rosenbrock} & \multicolumn{2}{c}{Goldstein} & \multicolumn{2}{c}{Himmelblau} \\
          & \multicolumn{1}{l}{Mean} & \multicolumn{1}{l}{std err} & \multicolumn{1}{l}{Mean} & \multicolumn{1}{l}{std err} & \multicolumn{1}{l}{Mean} & \multicolumn{1}{l}{std err} \\
    \hline
    P-X & 1.913 & \multicolumn{1}{r}{5.74E-05} & 1.234 & \multicolumn{1}{r}{5.82E-02} & 27.930 & \multicolumn{1}{r}{6.050} \\
    P-X ($\falsq=0.5$) & 1.628 & \multicolumn{1}{r}{1.09E-08} & 0.305 & \multicolumn{1}{r}{1.97E-08} & 17.651 & \multicolumn{1}{r}{1.48E-05} \\
    P-X ($\falsq=0.95$) & 1.628 & \multicolumn{1}{r}{1.09E-08} & 0.305 & \multicolumn{1}{r}{1.97E-08} & 17.651 & \multicolumn{1}{r}{1.48E-05} \\
    P-X ($\falsq=0.99$) & 1.628 & \multicolumn{1}{r}{1.09E-08} & 0.305 & \multicolumn{1}{r}{1.97E-08} & 17.651 & \multicolumn{1}{r}{1.48E-05} \\
    \hline
    Monte Carlo & \multicolumn{2}{c}{1.626} & \multicolumn{2}{c}{0.302} & \multicolumn{2}{c}{17.030} \\
    \bottomrule
    \end{tabular}%
  \label{tab:resVaolR10M500}%
\end{table}%

Finally, Table~\ref{tab:resVaolR20M500} shows the results obtained for the largest grid with $R=20, M=50$.

\begin{table}[H]
  \centering
  \caption{Falsification volume obtained with $R=20, M=500$. We estimate the falsifying volume as the volume of the remaining and violating hyperboxes (P-X). P-X ($\falsq$) algorithm variants use the Gaussian processes $\falsq$-quantiles to estimate the falsifying volume. Monte Carlo uses the same number of evaluations ($5000\times 30$) to perform a one-shot estimation of the falsifying volume, so no standard error is provided.}
    \begin{tabular}{lcccccc}
    \toprule
    \multicolumn{1}{c}{\multirow{2}[0]{*}{Algorithm}} & \multicolumn{2}{c}{Rosenbrock} & \multicolumn{2}{c}{Goldstein} & \multicolumn{2}{c}{Himmelblau} \\
          & \multicolumn{1}{l}{Mean} & \multicolumn{1}{l}{std err} & \multicolumn{1}{l}{Mean} & \multicolumn{1}{l}{std err} & \multicolumn{1}{l}{Mean} & \multicolumn{1}{l}{std err} \\
    \hline
    P-X & 1.909 & \multicolumn{1}{r}{5.86E-05} & 0.966 & \multicolumn{1}{r}{1.13E-03} & \multicolumn{1}{r}{25.078} & \multicolumn{1}{r}{0.220} \\
    P-X ($\falsq=0.5$) & 1.628 & \multicolumn{1}{r}{4.23E-08} & 0.304 & \multicolumn{1}{r}{6.53E-08} & 17.653 & \multicolumn{1}{r}{6.08E-06} \\
    P-X ($\falsq=0.95$) & 1.628 & \multicolumn{1}{r}{4.23E-08} & 0.304 & \multicolumn{1}{r}{6.53E-08} & 17.653 & \multicolumn{1}{r}{6.08E-06} \\
    P-X ($\falsq=0.99$) & 1.628 & \multicolumn{1}{r}{4.23E-08} & 0.304 & \multicolumn{1}{r}{6.53E-08} & 17.653 & \multicolumn{1}{r}{6.08E-06} \\
    \hline
    Monte Carlo & \multicolumn{2}{c}{1.626} & \multicolumn{2}{c}{0.302} & \multicolumn{2}{c}{17.030} \\
    \bottomrule
    \end{tabular}%
  \label{tab:resVaolR20M500}%
\end{table}

\subsection{Automated Requirement Falsification}\label{sec::falsification}

To run the Part-X function on falsification benchmarks, we developed an SBTG Python package%
\footnote{We plan to publicly release our open-source SBTG package after the review process concludes.} which implements a similar architecture to the S-Taliro \cite{Annpureddy2011STaLiRoAT} and BREACH \cite{donze10cav} Matlab tools. 
Our SBTG tool uses the Python packages RTAMT \cite{NickovicY2020atva} and TLTk \cite{CralleySHF2020rv} for computing the robustness of STL specifications.
The tool also supports a number of optimization functions besides Part-X.
Our SBTG tool requires an STL formula, a SUT object, and an options structure that modifies the behavior of the SBTG library. 
A SUT object handles the generation of signal inputs and manages the execution of the underlying SUT. 
When a vector $\mbf{x}$ is provided by Part-X (or any other optimizer) to the SUT object, the vector $\mbf{x}$ is first separated into static model parameters, e.g., initial conditions and/or other parameters, and signals (time-varying parameters). 
A sequence of points that corresponds to a time-varying parameter is interpolated by a user selected function to produce an input signal for the SUT.
Then, the SBTG executes the SUT using the provided parameters and signals, receives the SUT output trajectories, and computes the specification robustness.
At that point, the robustness value is returned to the Part-X function, and the process repeats until the Part-X algorithm has reached one of its terminating conditions or the maximum allowed testing budget.

Even though the ARCH falsification competition \cite{ernst2020arch} uses a number of benchmark problems, virtually all of them are Matlab/Simulink models.
In order to demonstrate the Part-X algorithm and our SBTG tool, we selected the F-16 benchmark (version 88ABW-2020-2188)~\cite{heidlauf2018verification}  from the ARCH competition.
The F16 benchmark provides both a Matlab/Simulink and a Python version of a simplified F16 Ground Control Avoidance System (GCAS). 
The GCAS system uses 16 continuous variables and piece-wise non-linear differential equations to perform autonomous maneuvers to avoid hitting the ground. 
Our benchmark instance defines three static inputs $\phi$, $\theta$ and $\psi$, which are the initial roll, pitch and yaw angles of the aircraft, and there are no time-varying inputs.
The three angles can range in the intervals $[0.2\pi, 0.2833\pi]$, $[-0.5\pi, -0.54\pi]$, and $[0.25\pi, 0.375\pi]$, respectively. 
The GCAS model returns the altitude of the aircraft over the time horizon $[0, 15]$.
The correctness requirement was the STL formula ``$\mathtt{Always}_{[0,15]} (altitude > 0)$'' which simply states that the altitude should always be above 0 during the first 15 sec.

Part-X was executed on the F16 GCAS model with the initial altitude set to various altitude values between 2300 ft. to 2400 ft., which are listed in Table~\ref{tab:f16benchmarkres}. These experiments were run on a server with an Intel(R) Xeon(R) Platinum 8260 @ 2.40GHz CPU with 128G RAM running Ubuntu 20.04.2 LTS. 
The F16 GCAS model is know to be falsifiable at altitudes 2300 ft. and completely non-falsifiable at altitude 2400 ft. We executed $250000$ simulations using uniform random sampling and confirmed that no simulation violated the safety requirement.

Part-X was run with initialization budget $n_0 = 30$, unclassified subregions, per-subregion, budget $n_{\mbox{\tiny{BO}}} = 10$, classified subregions sampling budget $n_c=100$, maximum budget $T=5000$, number of Monte Carlo iterations $R=20$, number of evaluations per iterations $M=500$, number of cuts $B = 2$, and classification level $\Significance = 0.05$. Finally, the minimum volume condition was set as $\minVol = 0.001$.

Table~\ref{tab:f16benchmarkres} shows the results of running Part-X on the F-16 benchmark. Each benchmark instance was run 50 times since Part-X is a stochastic algorithm. $\bar{FR}$ represents the falsification rate, i.e., out of 50 macro-replications, how many times Part-X was able to sample at least one falsification. Finally, $\min\rho_\phi$ is the minimum robustness value sampled across all the evaluations by the algorithm.
The estimation of the $q$-quantile falsification volume is reported under $\FalsVol_{q}$.

\begin{table*}[t]
  \centering
  \caption{F-16 benchmark results. The MC estimates were obtained with 250,000 samples generated uniformly. Confidence bounds are evaluated out of $50$ macro-replications using $95$\% confidence.}
    \begin{tabular}{l|l|l|l|l|l|l|l}
    \toprule
    \multicolumn{1}{l|}{\multirow{2}[0]{*}{$H_0$}} & \multicolumn{1}{l|}{\multirow{2}[0]{*}{Algorithm}} & \multicolumn{1}{c|}{\multirow{2}[0]{*}{$\bar{\mbox{FR}}$}} & \multicolumn{4}{c|}{$\FalsVol_q$} & \multicolumn{1}{l}{\multirow{2}[0]{*}{$\min\rho_\phi$}} \\
    \cline{4-7}
          &       &       & \multicolumn{1}{p{4.215em}|}{mean} & \multicolumn{1}{p{4.215em}|}{std\_error} & \multicolumn{1}{p{4.215em}|}{LCB} & \multicolumn{1}{p{4.215em}|}{UCB} &  \\
    \hline
    \multicolumn{1}{r|}{\multirow{4}[0]{*}{2300.0}} & Part - X ($q=0.50$) & 50    & 0.000687 & 1.19E-12 & 0.000671 & 0.000702 & -38.4581 \\
          & Part - X ($q=0.95$) & 50    & 0.000687 & 1.19E-12 & 0.000671 & 0.000702 & -38.4581 \\
          & Part - X ($q=0.99$) & 50    & 0.000687 & 1.19E-12 & 0.000671 & 0.000702 & -38.4581 \\
          & Uniform Random & 50    & 0.000679 & 8.68E-11 & 0.00055 & 0.000808 & -38.3656 \\
    \hline
    \multicolumn{1}{r|}{\multirow{4}[0]{*}{2338.0}} & Part - X ($q=0.50$) & 35    & 4.89E-06 & 2.62E-13 & 0.00E+00 & 1.20E-05 & -0.44906 \\
          & Part - X ($q=0.95$) & 35    & 4.89E-06 & 2.62E-13 & 0.00E+00 & 1.20E-05 & -0.44906 \\
          & Part - X ($q=0.99$) & 35    & 4.89E-06 & 2.62E-13 & 0.00E+00 & 1.20E-05 & -0.44906 \\
          & Uniform Random & 15    & 2.33E-06 & 3.32E-13 & 0.00E+00 & 1.03E-05 & -0.32952 \\
    \hline
    \multicolumn{1}{r|}{\multirow{4}[0]{*}{2338.2}} & Part - X ($q=0.50$) & 23    & 4.17E-06 & 1.96E-13 & 0.00E+00 & 1.03E-05 & -0.24887 \\
          & Part - X ($q=0.95$) & 23    & 4.17E-06 & 1.96E-13 & 0.00E+00 & 1.03E-05 & -0.24887 \\
          & Part - X ($q=0.99$) & 23    & 4.17E-06 & 1.96E-13 & 0.00E+00 & 1.03E-05 & -0.24887 \\
          & Uniform Random & 6     & 7.75E-07 & 8.99E-14 & 0.00E+00 & 4.93E-06 & -0.12933 \\
    \hline
    \multicolumn{1}{r|}{\multirow{4}[0]{*}{2338.4}} & Part - X ($q=0.50$) & 4     & 3.34E-06 & 1.83E-13 & 0.00E+00 & 9.26E-06 & -0.04059 \\
          & Part - X ($q=0.95$) & 4     & 3.34E-06 & 1.83E-13 & 0.00E+00 & 9.26E-06 & -0.04059 \\
          & Part - X ($q=0.99$) & 4     & 3.34E-06 & 1.83E-13 & 0.00E+00 & 9.26E-06 & -0.04059 \\
          & Uniform Random & 0     & 0     & 0     & -     & -     & 0.070859 \\
    \hline
    \multicolumn{1}{r|}{\multirow{4}[0]{*}{2338.5}} & Part - X ($q=0.50$) & 0     & 3.06E-06 & 1.69E-13 & 0.00E+00 & 8.77E-06 & 0.059504 \\
          & Part - X ($q=0.95$) & 0     & 3.06E-06 & 1.69E-13 & 0.00E+00 & 8.77E-06 & 0.059504 \\
          & Part - X ($q=0.99$) & 0     & 3.06E-06 & 1.69E-13 & 0.00E+00 & 8.77E-06 & 0.059504 \\
          & Uniform Random & 0     & 0     & 0     & -     & -     & 0.170954 \\
    \hline
    \multicolumn{1}{r|}{\multirow{4}[0]{*}{2338.6}} & Part - X ($q=0.50$) & 0     & 2.87E-06 & 1.73E-13 & 0.00E+00 & 8.64E-06 & 0.159608 \\
          & Part - X ($q=0.95$) & 0     & 2.87E-06 & 1.73E-13 & 0.00E+00 & 8.64E-06 & 0.159608 \\
          & Part - X ($q=0.99$) & 0     & 2.87E-06 & 1.73E-13 & 0.00E+00 & 8.64E-06 & 0.159608 \\
          & Uniform Random & 0     & 0     & 0     & -     & -     & 0.271049 \\
    \hline
    \multicolumn{1}{r|}{\multirow{4}[0]{*}{2350.0}} & Part - X ($q=0.50$) & 0     & 4.04E-09     & 1.01E-17     & 0.00E+00     & 4.80E-08     & 11.61762 \\
          & Part - X ($q=0.95$) & 0     & 4.04E-09     & 1.01E-17     & 0.00E+00     & 4.80E-08     & 11.61762 \\
          & Part - X ($q=0.99$) & 0     & 4.04E-09     & 1.01E-17     & 0.00E+00     & 4.80E-08     & 11.61762 \\
          & Uniform Random & 0     & 0     & 0     & -     & -     & 11.68186 \\
    \hline
    \multicolumn{1}{r|}{\multirow{4}[0]{*}{2400.0}} & Part - X ($q=0.50$) & 0     & 0     & 0     & -     & -     & 61.67597 \\
          & Part - X ($q=0.95$) & 0     & 0     & 0     & -     & -     & 61.67597 \\
          & Part - X ($q=0.99$) & 0     & 0     & 0     & -     & -     & 61.67597 \\
          & Uniform Random & 0     & 0     & 0     & -     & -     & 61.72921 \\
    \bottomrule
    \end{tabular}%
  \label{tab:f16benchmarkres}%
\end{table*}%

Two observations are immediate from Table~\ref{tab:f16benchmarkres}.
The quantile falsification volumes of instance 1 and 2 (altitude 2300, 2338) are conclusively indicating the presence of falsifying regions (non-zero volume).
These regions can be identified in the partition (see Figure~\ref{fig:f16_alt_2330_falsification}), and if desired, they could be avoided in the  operating space of the system in order to guarantee safety with respect to the requirement.
Alternatively, any debugging efforts can now be focused in these regions of the operating space where we have evidence that many falsifying behaviors exist.
More information can be extracted from these regions by using tools such as~\cite{DiwakaranST2017iccps}.

On the other hand, the quantile falsification volumes of instance 7 (altitude 2350) are inconclusive, i.e., the interval contains 0.
In other words, it is possible that the system is not falsifiable, i.e., it is safe with respect to the requirements.
At this point, it is up to the test engineer or the system evaluator to decide whether the volume (i.e., probability of falsification) can be tolerated for the target application. 
If the system evaluator determines that higher confidence is needed in the results, then the sampling process can be resumed.
The benefit of our approach is that the partitions contain information on which areas are the most likely to falsify and, hence, further sampling can be biased towards these partitions.

Figure~\ref{fig::f16_2330_partition_plot} presents some partition results for the F-16 benchmark. 
We note that in Figure~\ref{fig:f16_alt_2330_falsification}, we can see the undefined regions (in blue) since falsifying points have been identified.  
However, also in Figure~\ref{fig::f16_alt_2345_no_falsification}, while no falsifying points have been detected, still we have positive probability to falsify. 
\begin{figure}[H]
\centering
\subfigure[Partition generated by a falsifying run for altitude $H_0 = 2300$ \label{fig:f16_alt_2330_falsification}]{
\includegraphics[width=0.3\textwidth]{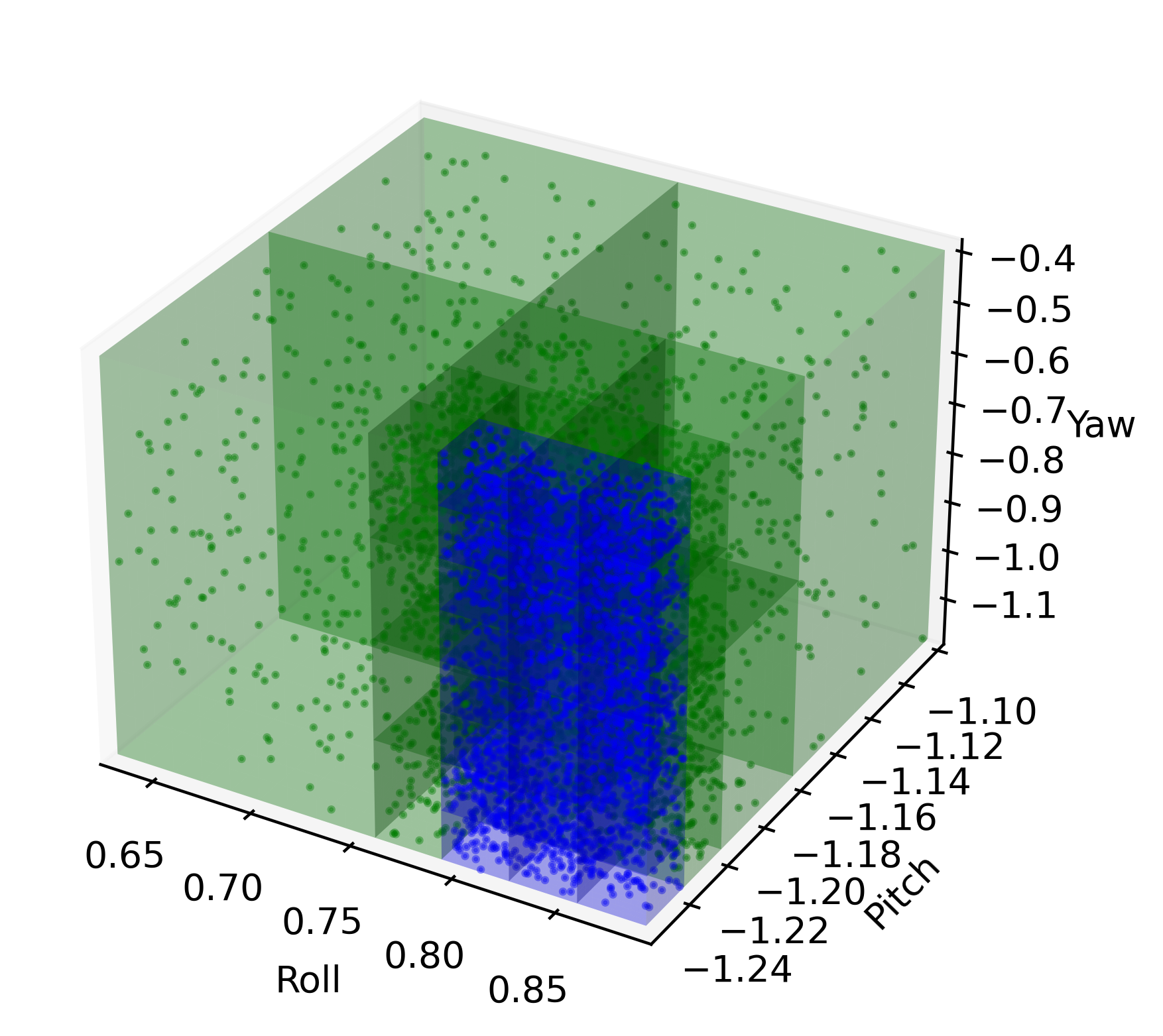}}\quad
\subfigure[The partition obtained by a non-falsifying run for altitude $H_0 = 2350$ \label{fig::f16_alt_2345_no_falsification}]{
\includegraphics[width=0.3\textwidth]{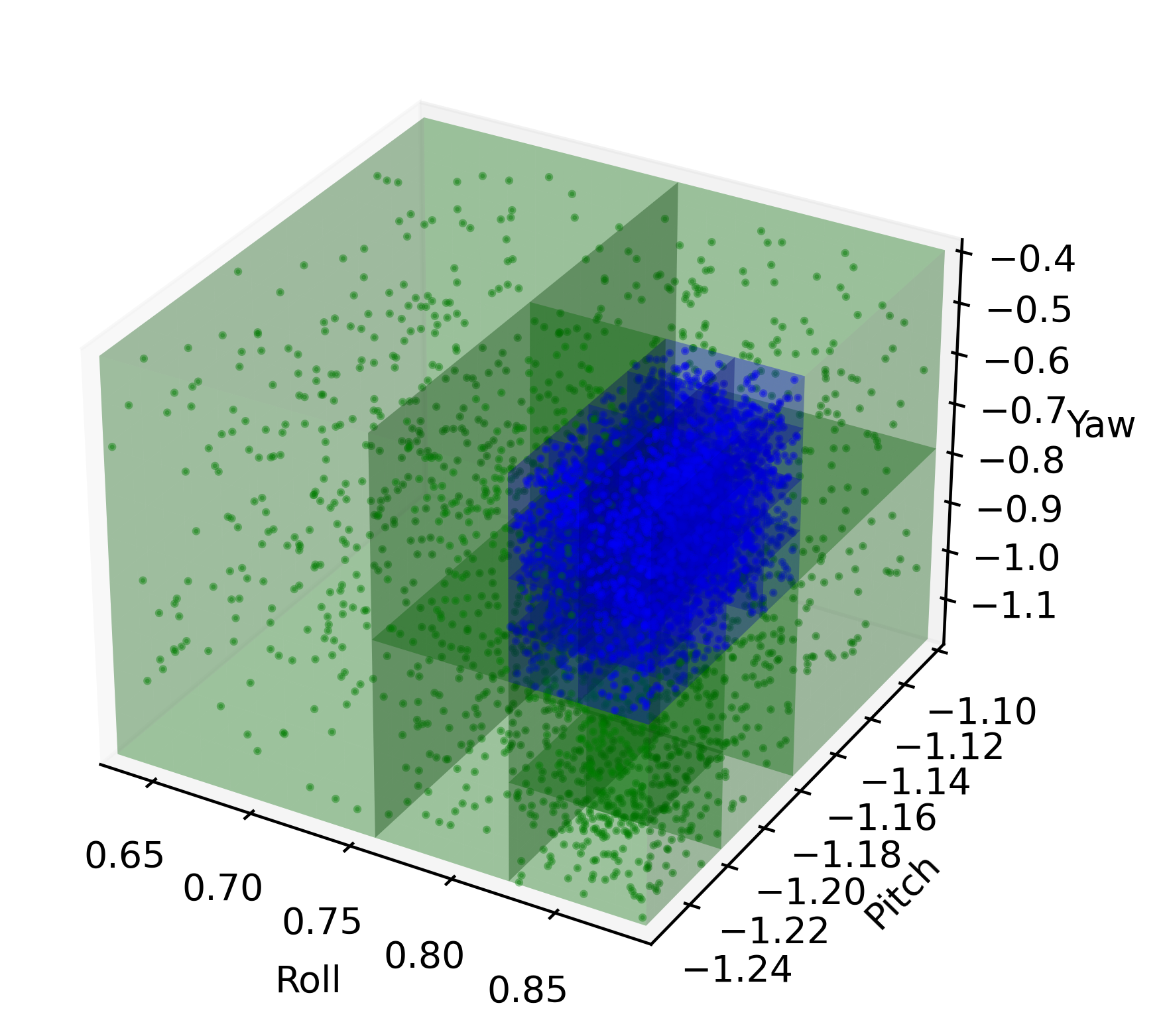}}
\caption{F-16 benchmark partition and sampling results for a single macro-replication.}\label{fig::f16_2330_partition_plot}
\end{figure}

\section{Conclusion}
We present, for the first time, the algorithm Part-X that generates test cases automatically with the objective of estimating the zero level-set of the robustness landscape of a Cyber-Physical System (CPS) induced by its safety requirements. 
Specifically, given an input space, a simulation tool (or test harness), a requirement, and a significance level, Part-X attempts to classify the input space into violating, satisfying, and undecided subregions.  
The algorithm relies on a collection of Gaussian processes, each defined over a single subregion, used to estimate the unknown robustness function in unsampled areas of the input. As a result, we provide the falsification volume to be interpreted as the probability that a falsifying input exists. 
In our main results, we show how to probabilistically bound the classification error at each iteration of the algorithm, thus providing a bound in the error of the estimation of the falsification volume.  
The numerical results demonstrate the ability of Part-X to estimate zero level-sets when applied to general nonlinear non-convex functions as well as when applied to the F16 benchmark. 
 The next step in our research is to extend our results from deterministic to stochastic systems as well as to explore scalability to higher dimensions.
 
 \section*{Acknowledgments}
This  work  is  supported  by  the  DARPA  ARCOS  program under contract FA8750-20-C-0507.

\bibliography{refs_pX}
\end{document}